\documentclass[11pt]{article}

\usepackage{microtype}
\usepackage{graphicx}
\usepackage{subfigure}
\usepackage{booktabs}

\usepackage{url}

\usepackage{amsmath,amsfonts,bm}

\def\Secref#1{Section~\ref{#1}}

\def\eqref#1{equation~\ref{#1}}
\def\Eqref#1{Equation~\ref{#1}}

\def\1{\bm{1}}

\def\vv{{\bm{v}}}

\DeclareMathAlphabet{\mathsfit}{\encodingdefault}{\sfdefault}{m}{sl}
\SetMathAlphabet{\mathsfit}{bold}{\encodingdefault}{\sfdefault}{bx}{n}

\def\gF{{\mathcal{F}}}
\def\gG{{\mathcal{G}}}
\def\gH{{\mathcal{H}}}

\def\gU{{\mathcal{U}}}

\def\gX{{\mathcal{X}}}

\def\sR{{\mathbb{R}}}

\newcommand{\R}{\mathbb{R}}

 \usepackage{enumerate}
\usepackage{enumitem}
\usepackage{bbm}

\usepackage{graphicx} \usepackage{caption}
\usepackage{amsmath}
\usepackage{amsthm}
\usepackage{amssymb}
\usepackage{tikz}
\usepackage{xcolor}
\usetikzlibrary{arrows}
\usepackage{bbm}
\usepackage{mathtools}

\usepackage{hyperref}

\allowdisplaybreaks

\newtheorem{assumption}{Assumption}[section]
\newtheorem{theorem}{Theorem}[section]
\newtheorem{inftheorem}{Informal Theorem}[section]
\newtheorem{corollary}{Corollary}[section]
\newtheorem{lemma}[theorem]{Lemma}
\newtheorem{claim}[theorem]{Claim}

\def\R{\mathbb{R}}

\def\cF{\mathcal{F}}
\def\cG{\mathcal{G}}

\newcommand{\simplex}{\triangle}
\newcommand{\abs}[1]{\left|#1\right|}
\newcommand{\expect}{\mathbb{E}}

\newcommand{\indict}{\mathbb{I}}

\newcommand{\states}{\mathcal{S}}
\newcommand{\trans}{P}

\newcommand{\actions}{\mathcal{A}}

\newcommand{\ex}{\mathop{\mathbb{E}}\limits}
\newcommand{\ep}{\mathop{\mathbb{E}}}

\newcommand{\relu}[1]{\sigma\left(#1\right)}

\newcommand{\mdp}{\mathcal{M}}

\newcommand{\bx}{{\bf x}}
\newcommand{\bX}{{\bf X}}
\newcommand{\bs}{{\bf s}}
\newcommand{\bS}{{\bf S}}
\newcommand{\ba}{{\bf a}}

\newcommand{\bpi}{{\boldsymbol{\pi}}}

\newenvironment{itemize*}{\begin{itemize}[leftmargin=*,topsep=0pt]\setlength{\itemsep}{0pt}\setlength{\parskip}{0pt}}{\end{itemize}}
\newenvironment{enumerate*}{\begin{enumerate}[leftmargin=*,topsep=0pt]\setlength{\itemsep}{0pt}\setlength{\parskip}{0pt}}{\end{enumerate}} 
\usepackage{hyperref}
\hypersetup{colorlinks=true,citecolor=blue,linkcolor=blue}
\usepackage[parfill]{parskip}
\usepackage{natbib}
\usepackage[margin=1in]{geometry}
\usepackage{authblk}

\begin{document}

\title{Provable Representation Learning for Imitation Learning via Bi-level Optimization}
\date{}
\author[1,2]{{\large Sanjeev Arora}}
\author[2]{\large Simon S. Du}
\author[3]{\large Sham Kakade}
\author[1]{\large Yuping Luo}
\author[1]{\large Nikunj Saunshi}

\affil[1]{\small Princeton University, Princeton, USA}
\affil[ ]{\texttt {\{arora, yupingl, nsaunshi\}@cs.princeton.edu}}
\affil[2]{\small Institute for Advanced Study, Princeton, USA}
\affil[ ]{\texttt {ssdu@ias.edu}}
\affil[3]{\small University of Washington, Seattle, USA}
\affil[ ]{\texttt {sham@cs.washington.edu}}
\maketitle

\begin{abstract}

A common strategy in modern learning systems is to learn a representation that is useful for many tasks, a.k.a. representation learning.
We study this strategy in the imitation learning setting for Markov decision processes (MDPs) where multiple experts' trajectories are available.
We formulate representation learning  as a bi-level optimization problem where the ``outer" optimization tries to learn the joint representation and the ``inner" optimization encodes the imitation learning setup and tries to learn task-specific parameters.
We instantiate this framework for the imitation learning settings of behavior cloning and observation-alone.
Theoretically, we show using our framework that representation learning can provide sample complexity benefits for imitation learning in both settings.
We also provide proof-of-concept experiments to verify our theory. \end{abstract}

\section{Introduction}
\label{sec:intro}

Humans can often learn from experts quickly and with a few demonstrations and we would like our artificial agents to do the same.
However, even for simple imitation learning tasks, the current state-of-the-art methods require thousand of demonstrations.
Humans do not learn new skills from scratch. 
We can summarize learned skills, distill them and build a common ground, a.k.a, representation that is useful for learning future skills.
Can we build an agent to do the same?

The current paper studies how to apply representation learning to imitation learning.
Specifically, we want our agent to be able to learn a representation from multiple experts' demonstrations, where the experts aim to solve different Markov decision processes (MDPs) that share the same state and action spaces but can differ in the transition and reward functions.
The agent can use this representation to reduce the number of demonstrations required for a new imitation learning task.
While several methods have been proposed~\citep{duan2017one,finn2017one,james2018task} to build agents that can adapt quickly to new tasks, none of them, to our knowledge, give provable guarantees showing the benefit of using past experience.
Furthermore, they do not focus on learning a representation.
See Section~\ref{sec:rel} for more discussions.

In this work, we propose a framework to formulate this problem and analyze the statistical gains of representation learning for imitation learning.
The main idea is to use bi-level optimization formulation where the ``outer" optimization tries to learn the joint representation and the ``inner" optimization encodes the imitation learning setup and tries to learn task-specific parameters.
In particular, the inner optimization is flexible enough to allow the agent to interact with the environment.
This framework allows us to do a rigorous analysis to show provable benefits of representation learning for imitation learning.
With this framework at hand, we make the following concrete contributions:
\begin{itemize*}
\item We first instantiate our framework in the setting where the agent can observe experts' actions and tries to find a policy that matches the expert's policy, a.k.a, behavior cloning.
This setting can be viewed as a straightforward extension of multi-task representation learning for supervised learning~\citep{maurer2016benefit}.
We show in this setting that with sufficient number of experts (possibly optimizing for different reward functions), the agent can learn a representation that provably reduces the sample complexity for a new target imitation learning task.
\item Next, we consider a more challenging setting where the agent \emph{cannot} observe experts' actions but only their states, a.k.a., the observation-alone setting.
We set the inner optimization as a min-max problem inspired by \cite{sun2019provably}.
Notably, this min-max problem requires the agent to interact with the environment to collect samples.
We again show that with sufficient number of experts, the agent can learn a representation that provably reduces the sample complexity for a target task where the agent cannot observe actions from source and target experts.
\item We conduct experiments in both settings to verify our theoretical insights by learning a representation from multiple tasks using our framework and testing it on a new task from the same setting.
Additionally, we use these learned representations to learn a policy in the RL setting by doing policy optimization.
We observe that by learning representations the agent can learn a good policy with fewer samples than needed to learn a policy from scratch.
\end{itemize*}
The key contribution of our work is to connect existing literature on multi-task representation learning that deals with supervised learning \citep{maurer2016benefit} to single task imitation learning methods with guarantees \citep{syed2010reduction,ross2011reduction,sun2019provably}.
To our knowledge, this is the first work showing such guarantees for general losses that are not necessarily convex.
\paragraph{Organization:} In Section~\ref{sec:rel}, we review and discuss related work.
Section~\ref{sec:pre} reviews necessary concepts and describes the basic representation learning setup.
In Section~\ref{sec:bi_opt}, we formulate representation learning for imitation learning as a bi-level optimization problem and give an overview of the kind of results we prove.
In Section~\ref{sec:meta_bc}, we show our theoretical guarantees for behavior cloning, i.e., the case when the agent can observe experts' actions.
In Section~\ref{sec:meta_oo}, we discuss our theoretical result for the observation alone setting.
In Section~\ref{sec:exp}, we present our experimental results showing the benefit of representation learning for imitation learning via our framework.
We conclude in Section~\ref{sec:con} and defer technical proofs to appendix.

\section{Related Work}
\label{sec:rel}
Representation learning has shown its great power in various domains; see \citet{bengio2013representation} for a survey.
Theoretically, \citet{maurer2016benefit} studied the benefit representation learning for sample complexity reduction in the multi-task supervised learning setting.
Recently, \citet{arora2019theoretical} analyzed the benefit of representation learning via contrastive learning.
While these papers all build representations for the agent / learner, researchers also try to build representations about the environment / physical world~\citep{wu2017Learning}.

Imitation learning can help with sample efficiency of many problems~\citep{ross2010efficient,sun2017complete,daume2009search,chang2015learning,pan2018agile}.
Most existing work consider the setting where the learner can observe expert's action.
A general strategy is use supervised learning to learn a policy that maps the state to action that matches expert's behaviors.
The most straightforward one is behavior cloning~\citep{pomerleau1991efficient}, which we also study in our paper.
More advanced approaches have also been proposed~\citep{ross2011reduction,ross2014reinforcement,sun2018truncated}.
These approaches, including behavior cloning, often enjoy sound theoretical guarantees in the single task case.
Our work extends the theoretical guarantees of behavior cloning to the multi-task representation learning setting.

This paper also considers a more challenging setting, imitation learning from observation alone.
Though some model-based methods have been proposed~\citep{torabi2018behavioral,edwards2018imitation}, these methods lack theoretical guarantees.
Another line of work learns a policy that minimizes the difference between the state distributions induced by it and the expert policy, under certain distributional metric~\citep{ho2016generative}.
\citet{sun2019provably} gave a theoretical analysis to characterize the sample complexity of this approach and our method for this setting is inspired by their approach.

A line of work uses meta-learning for imitation learning~\citep{duan2017one,finn2017one,james2018task}.
Our work is different from theirs as we want to explicitly learn a representation that is useful across all tasks whereas these work try to learn a meta-algorithm that can quickly adapt to a new task.
For example, \citet{finn2017one} used a gradient based method for adaptation.
Recently \citet{raghu2019rapid} argued that most of the power of MAML \citep{finn2017model} like approaches comes from learning a shared representation.

On the theoretical side of meta-learning and multi-task learning, \citet{baxter2000model} performed the first theoretical analysis and gave sample complexity bounds using covering numbers.
\citet{maurer2009transfer} analyzed linear representation learning, while \citet{bullins2019generalize,denevi2018incremental} provide efficient algorithms to learn linear representations that can reduce sample complexity of a new task.
Another recent line of work analyzes gradient based meta-learning methods, similar to MAML \citep{finn2017model}.
Existing work on the sample complexity and regret of these methods \citep{denevi2019learning,finn2019online,khodak2019adaptive} show guarantees for convex losses by leveraging tools from online convex optimization.
In contrast, our analysis works for arbitrary function classes and the bounds depend on the Gaussian averages of these classes.
Recent work \citep{rajeswaran2019meta} uses a bi-level optimization framework for meta-learning and improves computation aspects of meta-learning through implicit differentiation; our interest lies in the statistical aspects.
 
\section{Preliminaries}
\label{sec:pre}
\paragraph{Markov Decision Processes (MDPs):}
Let $\mdp=\left(\states, \actions, \trans, C,\nu\right)$ be an MDP, where $\states$ is the state space, $\actions$ is the finite action space with $\abs{\actions}=K$, $H\in\mathbb{Z}_+$ is the planning horizon, $\trans: \states \times \actions \rightarrow \simplex\left(\states\right)$ is the transition function, $C:\states\times\actions\rightarrow\mathbb{R}$ is the cost function and $\nu\in\simplex(S)$ is the initial state distribution.
We assume that cost is bounded by $1$, i.e. $C(s,a)\le 1,\forall s\in\states, a\in\actions$.
This is a standard regularity condition used in many theoretical reinforcement learning work.
A (stochastic) policy is defined as $\bpi=(\pi_1,\dots,\pi_H)$, where $\pi_h:\states\rightarrow\simplex(\actions)$ prescribes a distribution over action for each state at level $h\in[H]$.
For a stationary policy, we have $\pi_1=\dots=\pi_H=\pi$.
A policy $\bpi$ induces a random trajectory $s_1,a_1,s_2,a_2,\dots,s_H,a_H$ where $s_1\sim\nu, a_1\sim\pi_1(s), s_2\sim P_{s_1,a_1}$ etc.
Let $\nu_h^\bpi$ denote the distribution over $\states$ induced at level $h$ by policy $\bpi$.
The value function $V_h^\bpi:\states\rightarrow\mathbb{R}$ is defined as 
\begin{align*}
	V_h^\bpi(s_h)=\ex_{}\left[\sum_{i=h}^{H}C(s_i,a_i)\mid a_i\sim\pi_i(s_i),s_{i+1}\sim P_{s_i,a_i}\right]
\end{align*}
and the state-action function $Q_h^\bpi(s_h,a_h)$ is defined as $Q_h^\bpi(s_h,a_h)=\expect_{s_{h+1}\sim P_{s_h,a_h}}\left[V^\bpi_h(s_{h+1})\right]$.
The goal is to learn a policy $\bpi$ that minimizes the {\em expected cost} $J(\bpi)=\expect_{s_1\sim\nu}V^\bpi_1(s_1)$.
We define the Bellman operator at level $h$ for any policy $\bpi$ as $\Gamma_h^\bpi:\mathbb{R}^\states\rightarrow\mathbb{R}^\states$, where for $s\in\states$ and $g\in\mathbb{R}^\states$,
\begin{align}
	\label{eq:bellman_op}
	(\Gamma_h^{\bpi}g)(s)\coloneqq\expect_{a\sim\pi_h(s),s'\sim P_{s,a}}[g(s')]
\end{align}

\paragraph{Multi-task Imitation learning:}
We formally describe the problem we want to study.
We assume there are multiple tasks (MDPs) sampled i.i.d. from a distribution $\eta$. 
A task $\mu\sim\eta$ is an MDP $\mdp_\mu=(\states,\actions,H,P_\mu,C_\mu,\nu_\mu)$; all tasks share everything except the cost function, initial state distribution and transition function.
For simplicity of presentation, we will assume a common transition function $P$ for all tasks; proofs remain exactly the same even otherwise.
For every task $\mu$, $\bpi^*_\mu=(\pi^*_{1,\mu},\dots,\pi^*_{H,\mu})$ is an {\em expert policy} that the learner has access to in the form of trajectories induced by that policy.
The trajectories may or may not contain expert's actions.
These correspond to two settings that we discuss in more detail in Section~\ref{sec:meta_bc} and Section~\ref{sec:meta_oo}.
The distributions of states induced by this policy at different levels are denoted by $\{\nu^*_{1,\mu},\dots,\nu^*_{H,\mu}\}$ and the average state distribution as $\nu^*_\mu=\frac{1}{H}\sum\limits_{h=1}^{H}\nu^*_{h,\mu}$.
We define $V^*_{h,\mu}$ to be the value function of $\bpi^*_{\mu}$ and $J_\mu$ to be the expected cost function for task $\mu$.
We will drop the subscript $\mu$ whenever the task at hand is clear from context.
Of interest is also the special case where the expert policy $\bpi^*_\mu$ is stationary.

\paragraph{Representation learning:}
In this work, we wish to learn policies from a function class of the form $\Pi=\cF\circ\Phi$, where $\Phi\subseteq\{\phi: \states\rightarrow\R^d\mid\|\phi(s)\|_2\le R\}$ is a class of bounded norm {\em representation functions} mapping states to vectors and $\cF\subseteq\{f:\R^d\rightarrow\Delta(\actions)\}$ is a class of functions mapping state representations to distribution over actions.
We will be using linear functions, i.e. $\cF=\{x\rightarrow \verb|softmax|(Wx)\ |\ W\in\R^{K\times d}, \|W\|_F\le 1\}$. We denote a policy parametrized by $\phi\in\Phi$ and $f\in\cF$ by $\pi^{\phi,f}$, where $\pi^{\phi,f}(a|s)=f(\phi(s))_a$.
In some cases, we may also use the policy $\pi^{\phi,f}(a|s)=\indict\{a=\arg\max\limits_{a'\in A}f(\phi(s))_{a'}\}$\footnote{Break ties in any way}.
Denote $\Pi^{\phi}=\{\pi^{\phi,f}: f\in\cF\}$ to be the class of policies that use $\phi$ as the representation function.

Given demonstrations from expert policies for $T$ tasks sampled independently from $\eta$, we wish to first learn representation functions $(\hat\phi_1,\dots,\hat{\phi}_H)$ so that we can use a few demonstrations from an expert policy $\bpi^*$ for new task $\mu\sim\eta$ and learn a policy $\bpi=(\pi_1,\dots,\pi_H)$ that uses the learned representations, i.e. $\pi_h\in\Pi^{\hat\phi_h}$, such that has average cost of $\bpi$ is not too far away from $\bpi^*$.
In the case of stationary policies, we need to learn a single $\phi$ by using tasks and learn $\pi\in\Pi^{\phi}$ for a new task.
The hope is that data from multiple tasks can be used to learn a complicated function $\phi\in\Phi$ first, thus requiring only a few samples for a new task to learn a linear policy from the class $\Pi^\phi$.

\paragraph{Gaussian complexity:}
As in \cite{maurer2016benefit}, we measure the complexity of a function class $\gH\subseteq\{h:\gX\rightarrow\sR^d\}$ on a set $\bX=(X_1,\dots,X_n)\in\gX^n$ by using the following Gaussian average
\begin{align}\label{eq:gaussian_avg}
G(\gH(\bX))=\expect\left[\sup\limits_{h\in\gH}\sum\limits_{\substack{i=1\\j=1}}^{d,n}\gamma_{ij}h_i(X_j)\mid X_j\right]
\end{align}
where $\gamma_{ij}$ are independent standard normal variables.
\citet{bartlett2003rademacher} also used Gaussian averages to show some generalization bounds.
 
\section{Bi-level Optimization Framework}
\label{sec:bi_opt}
In this section we introduce our framework and give a high-level description of the conditions under which this framework gives us statistical guarantees.
Our main idea is to phrase learning representations for imitation learning as the following bi-level optimization
\begin{align}\label{eq:bi_opt}
\min\limits_{\phi\in\Phi}L(\phi)\coloneqq\ex_{\mu\sim\eta}\min\limits_{\pi\in\Pi^{\phi}}\ell^\mu(\pi)
\end{align}
Here $\ell^\mu$ is the {\em inner} loss function that penalizes $\pi$ being different from $\pi^*_\mu$ for the task $\mu$.
In general, one can use any loss $\ell^\mu$ that is used for single task imitation learning, e.g. for the behavioral cloning setting (cf. Section~\ref{sec:meta_bc}), $\ell^\mu$ is a classification like loss that penalizes the mismatch between predictions by $\pi^*$ and $\pi$, while for the observation-alone setting (cf. Section~\ref{sec:meta_oo}) it is some measure of distance between the state visitation distributions induced by $\pi$ and $\pi^*$.
The {\em outer} loss function is over the representation $\phi$.
The use of bi-level optimization framework naturally enforces policies in the inner optimization to share the same representation.

While Equation~\ref{eq:bi_opt} is formulated in terms of the distribution $\eta$, in practice we only have access to few samples for $T$ tasks; let $\bx^{(1)},\dots,\bx^{(T)}$ denote samples from tasks $\mu^{(1)},\dots,\mu^{(T)}$ sampled i.i.d. from $\eta$.
We thus learn the representation $\hat{\phi}$ by minimizing empirical version $\hat{L}$ of Equation~\ref{eq:bi_opt}.
\begin{align*}
\hat{L}(\phi)=\frac{1}{T}\sum\limits_{i=1}^T\min\limits_{\pi\in\Pi^{\phi}}\ell^{\bx^{(i)}}(\pi)=\frac{1}{T}\sum\limits_{i=1}^T\ell^{\bx^{(i)}}(\pi^{\phi,\bx^{(i)}})
\end{align*}
where $\ell^{\bx}$ is the empirical loss on samples $\bx$ and $\pi^{\phi,\bx}=\arg\min_{\pi\in\Pi^\phi}\ell^\bx(\pi)$ corresponds to a task specific policy that uses a fixed representation $\phi$.
Our goal then is to show that for a new task $\mu\sim\eta$, the learned representation can be used to learn a policy $\pi^{\hat{\phi},\bx}$ by using samples $\bx$ from the task $\mu$ that has low expected MDP cost $J_\mu$ (defined in \Secref{sec:pre})
\begin{inftheorem}
	\label{thm:informal}
With high probability over the sampling of train task data and with sufficient number of tasks and samples (expert demonstrations) per task, $\hat{\phi}=\arg\min_{\phi\in\Phi}\hat{L}(\phi)$ will satisfy
\begin{align*}
\ex_{\mu\sim\eta}\ex_{\bx}J_\mu(\pi^{\hat{\phi},\bx})-\ex_{\mu\sim\eta}J_\mu(\pi^*_\mu) \text{ is small }
\end{align*}
\end{inftheorem}

At a high level, in order to prove such a theorem for a particular choice of $\ell^\mu$, we would need to prove the following three properties about $\ell^\mu$ and $\ell^\bx$:
\begin{enumerate*}
\item $\ell^\bx(\pi)$ concentrates to $\ell^\mu(\pi)$ simultaneously for all $\pi\in\Pi^{\phi}$ (for a fixed $\phi$), with sample complexity depending on some complexity measure of $\Pi^{\phi}$ rather than being polynomial in $|\states|$;
\item if $\phi$ and $\phi'$ induce ``similar'' representations then $\min_{\pi\in\Pi^{\phi}}\ell^\mu(\pi)$ and $\min_{\pi\in\Pi^{\phi'}}\ell^\mu(\pi)$ are close;
\item a small value of $\ell^\mu(\pi)$ implies a small value for $J_\mu(\pi)-J_\mu(\pi^*_\mu)$.
\end{enumerate*} 
The first property ensures that learning a policy for a single task by fixing the representation is sample efficient, thus making representation learning a useful problem to solve.
The second property is specific to representation learning and requires $\ell^\mu$ to use representations in a smooth way.
This ensures that the empirical loss for $T$ tasks is a good estimate for the average loss on tasks sampled from $\eta$.
The third property ensures that matching the behavior of the expert as measured by the loss $\ell^\mu$ ensures low average cost i.e., $\ell^\mu$ is meaningful for the average cost; any standard imitation learning loss will satisfy this.
We {\em prove} these three properties for the cases where $\ell^\mu$ is the either behavioral cloning loss or observation-alone loss, with natural choices for the empirical loss $\ell^\bx$.
However the general proof recipe can be used for potentially many other settings and loss functions.

In the next section, we will describe representation learning for behavioral cloning as an instantiation of the above framework and describe the various components of the framework.
Furthermore we will describe the results and give a proof sketch to show how the aforementioned properties help us show our final guarantees.
The guarantees for this setting follow almost directly from results in \cite{maurer2016benefit} and \cite{ross2011reduction}.
Later in Section~\ref{sec:meta_oo} we describe the same for the observations alone setting which is more non-trivial. 
\section{Representation Learning for Behavioral Cloning}
\label{sec:meta_bc}
\paragraph{Choice of $\ell^\mu$:}
We first specify the inner loss function in the bi-level optimization framework.
In the single task setting, the goal of behavioral cloning \citep{syed2010reduction,ross2011reduction} is to use expert trajectories of the form $\tau=(s_1,a_1,\dots,s_H,a_H)$ to learn a stationary policy\footnote{We can easily extend the theory to non-stationary policies} that tries to mimic the decisions of the expert policy on the states visited by the expert.
For a task $\mu$, this reduces to a supervised classification problem that minimizes a surrogate to the following loss $\ell^\mu_{0-1}(\pi)=\expect_{s\sim\nu^*_\mu,a\sim\pi^*_\mu(s)}\indict\{\pi(s)\neq a\}$.
We abuse notation and denote this distribution over $(s,a)$ for task $\mu$ as $\mu$; so $(s,a)\sim\mu$ is the same as $s\sim\nu^*_\mu,a\sim\pi^*_\mu(s)$. 
Prior work \citep{syed2010reduction,ross2011reduction} have shown that a small value of $\ell^\mu_{0-1}(\pi)$ implies a small difference $J(\pi)-J(\pi^*)$.
Thus for our setting, we choose $\ell^\mu$ to be of the following form
\begin{align}\label{eq:ell_mu_bc}
\ell^\mu(\pi)=\ex_{s\sim\nu^*_\mu,a\sim\pi^*_\mu(s)}\ell(\pi(s),a)=\ex_{(s,a)\sim\mu}\ell(\pi(s),a)
\end{align}
where $\ell$ is any surrogate to 0-1 loss $\indict\{a\neq\arg\max\limits_{a'\in A}\pi(s)_{a'}\}$ that is  {\em Lipschitz} in $\phi(s)$.
In this work we consider the logistic loss $\ell(\pi(s),a)=-\log(\pi(s)_a)$.
\paragraph{Learning $\phi$ from samples:}
Given expert trajectories for $T$ tasks $\mu^{(1)},\dots,\mu^{(T)}$ we construct a dataset $\bX=\{\bx^{(1)},\dots,\bx^{(T)}\}$, where $\bx^{(t)}=\{(s^t_j,a^t_j)\}_{j=1}^{n}\sim(\mu^{(t)})^n$ is the dataset for task $t$.
Details of the dataset construction are provided in Section~\ref{subsec:dataset_bc}.
Let $\bS$ denote the set of states $\{s^t_j\}$.
Instantiating our framework, we learn a good representation by solving $\hat{\phi}=\arg\min\limits_{\phi\in\Phi}\hat{L}(\phi)$, where
\begin{align}\label{eq:L_hat_bc}
	\hat{L}(\phi)&\coloneqq\frac{1}{T}\sum\limits_{t=1}^{T}\min\limits_{\pi\in\Pi^{\phi}}\frac{1}{n}\sum\limits_{j=1}^n\ell(\pi(s^t_j),a^t_j)\nonumber\\
	&=\frac{1}{T}\sum\limits_{t=1}^{T}\min\limits_{\pi\in\Pi^{\phi}}\hat{\ell}^{\bx^{(t)}}(\pi)
\end{align}
where $\ell^\bx$ is loss on samples $\bx=\{(s_j,a_j)\}_{j=1}^n$ defined as $\ell^\bx(\pi)=\frac{1}{n}\sum_{j=1}^n\ell(\pi(s_j),a_j)$.
\paragraph{Evaluating representation $\hat{\phi}$:}
A learned representation $\hat{\phi}$ is tested on a new task $\mu\sim\eta$ as follows: draw samples $\bx\sim\mu^n$ using trajectories from $\pi^*_\mu$ and solve $\pi^{\hat{\phi},\bx}=\arg\min\limits_{\pi\in\Pi^{\hat{\phi}}}\hat{\ell}^\bx(\pi)$.
Does $\pi^{\hat{\phi},\bx}$ have expected cost $J_\mu(\pi^{\hat{\phi},\bx})$ not much larger than $J_\mu(\pi^*_\mu)$?
The following theorem answers this question.
We make the following two assumptions to prove the theorem.
\begin{assumption}
	\label{ass:deterministic_policy}
	The expert policy $\pi^*_\mu$ is deterministic for every $\mu\in\text{support}(\eta)$.
\end{assumption}
\begin{assumption}[Policy realizability]
	\label{ass:realizability_policy}
	There is a representation $\phi^*\in\Phi$ such that for every $\mu\in\text{support}(\eta)$, $\pi_\mu\in\Pi^{\phi^*}$ such that $\pi_\mu(s)_{\pi^*_\mu(s)}\footnote{We abuse notation and use $\pi^*_\mu(s)$ instead of $\arg\max\limits_{a\in\actions}\pi^*_\mu(s)_a$}\ge1-\gamma, \forall s\in\states$ for some $\gamma<1/2$.
\end{assumption}
The first assumption holds if $\pi^*_\mu$ is aiming to maximize some cost function.
 The second assumption is for representation learning to make sense:  we need to assume the existence of a common representation $\phi^*$ that can approximate all expert policies and $\gamma$ measures this expressiveness of $\Phi$.
Now we present our first main result about the performance of the learned representation on a new imitation learning task $\mu$, whose performance is measure by the average cost $J_\mu$.
\begin{theorem}
	\label{thm:meta_bc}
	Let $\hat{\phi}\in\arg\min\limits_{\phi\in\Phi}\hat{L}(\phi)$. Under Assumptions~\ref{ass:deterministic_policy},\ref{ass:realizability_policy}, with probability $1-\delta$ over the sampling of dataset $\bX$, we have
\begin{align*}
\ex_{\mu\sim\eta}\ex_{\bx\sim\mu^n}J_\mu(\pi^{\hat{\phi},\bx})-\ex_{\mu\sim\eta}J_\mu(\pi^*_\mu)\le H^2(2\gamma+\epsilon_{gen})
\end{align*}
where $\epsilon_{gen}=c\frac{G(\Phi(\bS))}{T\sqrt{n}}+c'\frac{R\sqrt{K}}{\sqrt{n}}+c''\sqrt{\frac{\ln(4/\delta)}{T}}$, for some small constants $c,c',c''$.
\end{theorem}

To gain intuition for what the above result means, we give a PAC-style guarantee for the special case where the class of representation functions $\Phi$ is finite.
This follows directly from the above theorem and the use of Massart's lemma.
\begin{corollary}
In the same setting as Theorem~\ref{thm:meta_bc}, suppose $\Phi$ is finite. If number of tasks satisfies $T\ge c_1\max\left\{\frac{H^4R^2\log(|\Phi|)}{\epsilon^2}, \frac{H^4\ln(4/\delta)}{\epsilon^2}\right\}$, and number of samples (expert trajectories) per task satisfies $n\ge c_2\frac{H^4R^2K}{\epsilon^2}$ for small constants $c_1,c_2$, then with probability $1-\delta$,
\begin{align*}
\ex_{\mu\sim\eta}\ex_{\bx\sim\mu^n}J_\mu(\pi^{\hat{\phi},\bx})-\ex_{\mu\sim\eta}J_\mu(\pi^*_\mu)\le H^2\gamma + \epsilon
\end{align*}
\end{corollary}
\paragraph{Discussion:}
The above bound says that as long as we have enough tasks to learn a representation from $\Phi$ and sufficient samples per task to learn a linear policy, the learned policy will have small average cost on a new task from $\eta$.
The first term $H^2\gamma$ is small if the representation class $\Phi$ is expressive enough to approximate the expert policies (see Assumption~\ref{ass:realizability_policy}).
The results says that if we have access to data from $T=O\left(\frac{H^4R^2\log(|\Phi|)}{\epsilon^2}\right)$ tasks sampled from $\eta$, we can use them to learn a representation such that for a new task we only need $n=O\left(\frac{H^4R^2K}{\epsilon^2}\right)$ samples (expert demonstrations) to learn a linear policy with good performance.
In contrast, without access to tasks, we would need $n=O\left(\max\left\{\frac{H^4R^2\log(|\Phi|)}{\epsilon^2}, \frac{H^4R^2K}{\epsilon^2}\right\}\right)$ samples from the task to learn a good policy $\pi\in\Pi$ from scratch.
Thus if the complexity of the representation function class $\Phi$ is much more than number of actions ($\log(|\Phi|) \gg K$ in this case), then multi-task representation learning might be much more sample efficient\footnote{These statements are qualitative since we are comparing upper bounds.}.
Note that the dependence of sample complexity on $H$ comes from the error propagation when going from $\ell^\mu$ to $J_\mu$; this is also observed in single task imitation learning \citep{ross2011reduction,sun2019provably}.

We give a proof sketch for Theorem~\ref{thm:meta_bc} below, while the full proof is deferred to \Secref{apdx:bc_proofs}.

\subsection{Proof sketch}\label{subsubsec:proof_sketch_bc}
The proof has two main steps.
In the first step we bound the error due to use of samples.
The policy $\pi^{\phi,\bx}$ that is learned on samples $\bx\sim\mu^n$ is evaluated on the distribution $\mu$ and the average loss incurred by representation $\phi$ across tasks is $\bar{L}(\phi)=\ex_{\mu\sim\eta}\ex_{\bx\sim\mu^n}\ell^\mu(\pi^{\phi,\bx})$.

On the other hand, if the learner had complete access to the distribution $\eta$ and distributions $\mu$ for every task, then the loss minimizer would be $\phi^*=\arg\min_{\phi\in\Phi}L(\phi)$, where $L(\phi)\coloneqq\ex_{\mu\sim\eta}\min\limits_{\pi\in\Pi^{\phi}}\ell^\mu(\pi)$.
Using results from \cite{maurer2016benefit}, we can prove the following about $\hat{\phi}$
\begin{lemma}
	\label{lem:gen_bc}
	With probability $1-\delta$ over the choice of $\bX$, $\hat{\phi}\in\arg\min\limits_{\phi\in\Phi}\hat{L}(\phi)$ satisfies
\begin{align*}
\bar{L}(\hat{\phi})\le \min\limits_{\phi\in\Phi}L(\phi)+c\frac{G(\Phi(\bS))}{T\sqrt{n}}+c'\frac{R\sqrt{K}}{\sqrt{n}}+c''\sqrt{\frac{\ln(1/\delta)}{T}}
\end{align*}
\end{lemma}
The proof of this lemma is provided in \Secref{apdx:bc_proofs}.

The second step of the proof is connecting the loss $\bar{L}(\phi)$ and the average cost $J_\mu$ of the policies induced by $\phi$ for tasks $\mu\sim\eta$.
This can obtained by using the connection between the surrogate 0-1 loss $\ell^\mu$ and the cost $J_\mu$ that has been established in prior work \citep{ross2011reduction,syed2010reduction}.
The following lemma uses the result for deterministic expert policies from \citet{ross2011reduction}.
\begin{lemma}
	\label{lem:j_mu_bc}
	Given a representation $\phi$ with $\bar{L}(\phi)\le\epsilon$. Let $\bx\sim\mu^n$ be samples for a new task $\mu\sim\eta$. Let $\pi^{\phi,\bx}$ be the policy learned by behavioral cloning on the samples, then under Assumption~\ref{ass:deterministic_policy}\begin{align*}
\ex_{\mu\sim\eta}\ex_{\bx\sim\mu^n}J_\mu(\pi^{\phi,\bx})-\ex_{\mu\sim\eta}J_\mu(\pi^*_\mu)\le H^2\epsilon
\end{align*}
\end{lemma}
This suggests that representations with small $\bar{L}$ do well on the imitation learning tasks.
A simple implication of Assumption~\ref{ass:realizability_policy} that $\min_{\phi\in\Phi}L(\phi)\le L(\phi^*)\le\gamma$, along with the above two lemmas completes the proof.

\section{Representation Learning for Observations Alone Setting}
\label{sec:meta_oo}
Now we consider the setting where we cannot observe experts' actions but only their states.
As in \citet{sun2019provably}, we also solve a problem at each level; consider a level $h\in[H]$.
\paragraph{Choice of $\ell^\mu_h$:}
Let $\bpi^*_\mu=\{\pi^*_{1,\mu},\dots,\pi^*_{H,\mu}\}$ be the sequence of expert policies (possibly stochastic) at different levels for the task $\mu$.
Let $\nu^*_{h,\mu}$ be the distribution induced on the states at level $h$ by the expert policy $\bpi^*_\mu$.
The goal in imitation learning with observations alone \citep{sun2019provably} is to learn a policy $\bpi=(\pi_1,\dots,\pi_H)$ that matches the distributions $\nu^\pi_{h}$ with $\nu^*_{h}$ for every $h$, w.r.t. a discriminator class $\cG$\footnote{If $\cG$ contains all bounded functions, then it reduces to minimizing TV between $\nu^\pi_h$ and $\nu^*_h$.} that contains the true value functions $V^*_1,\dots,V^*_H$ and is {\em approximately} closed under the Bellman operator of $\bpi^*$.
Instead, in this work we learn $\bpi$ that matches the distributions $\pi_h\cdot\nu^*_{h}$\footnote{The sampling $s\sim\pi_h\cdot\nu^*_{h}$ is defined as sampling $s'\sim\nu^*_{h}, a\sim\pi_h(s'), s\sim P_{s',a}$.} and $\nu^*_{h+1}$ for every $h$ w.r.t. to a class $\cG\subseteq\{g:\states\rightarrow\mathbb{R},|g|_\infty\le1\}$ that contains the value functions and has a stronger Bellman operator closure property.
For every task $\mu$, $\ell^\mu_h$ is defined as
\begin{align}\label{eq:ell_mu_oo}
\ell^\mu_h(\pi)&=\max\limits_{g\in\cG}[\ex_{s\sim\nu^*_{h,\mu}}\ex_{\substack{a\sim\pi(s)\\\tilde{s}\sim P_{s,a}}}g(\tilde{s})-\ex_{\bar{s}\sim\nu^*_{h+1,\mu}} g(\bar{s})]\\\nonumber
&=\max\limits_{g\in\cG}[\ex_{s\sim\nu^*_{h,\mu}}\ex_{\substack{a\sim\mathcal{U}(\actions)\\\tilde{s}\sim P_{s,a}}}K\pi(a|s)g(\tilde{s})-\ex_{\bar{s}\sim\nu^*_{h+1,\mu}} g(\bar{s})]
\end{align}
where we rewrite $\ell^\mu_h$ by importance sampling in the second equation; this will be useful to get an empirical estimate.
While our definition of $\ell^\mu_h$ differs slightly from the one used in \citet{sun2019provably}, using similar techniques, we will show that small values for $\ell^\mu_h(\pi_h)$ for every $h\in[H]$ will ensure that the policy $\bpi=(\pi_1,\dots,\pi_H)$ will have expected cost $J_\mu(\bpi)$ close to $J_\mu(\bpi^*_\mu)$.
We abuse notation, and for a task $\mu$ we denote $\mu=(\mu_1,\dots,\mu_H)$ where $\mu_h$ is the distribution of $(s,a,\tilde{s},\bar{s})$ used in $\ell^\mu_h$; thus $(s,a,\tilde{s},\bar{s})\sim\mu_h$ is equivalent to $s\sim\nu^*_{h,\mu}, a\sim\gU(\actions), \tilde{s}\sim P_{s,a}, \bar{s}\sim\nu^*_{h+1,\mu}$.

\paragraph{Learning $\phi_h$ from samples:}
We assume, 1)  access to $2n$ expert trajectories for $T$ independent {\em train} tasks, 2) ability to reset the environment at any state $s$ and sample from the transition $P(\cdot|s,a)$ for any $a\in \actions$.
The second condition is satisfied in many problems equipped with simulators.
Using the sampled trajectories for the $T$ tasks $\{\mu^{(1)},\dots,\mu^{(T)}\}$ and doing some interaction with environment, we get the following dataset $\bX=\{\bX_1,\dots,\bX_H\}$ where $\bX_h$ is the dataset for level $h$.
Specifically, $\bX_h=\{\bx^{(1)}_h,\dots,\bx^{(T)}_H\}$ where $\bx^{(i)}_h=\{(s^i_j,a^i_j,\tilde{s}^i_j,\bar{s}^i_j)\}_{j=1}^{n}\sim(\mu^{(i)})^n$ is the dataset for task $i$ at level $h$.
Additionally we denote $\bS_h=\{s_j^i\}_{i=1,j=1}^{T,n}$ to be all the $s$-states in $\bX_h$, $\tilde{\bS}_h$ and $\bar{\bS}_h$ are similarly defined as the collections of all the $\tilde{s}$-states and $\bar{s}$-states respectively.
Details about how this dataset is constructed from expert trajectories and interactions with environment is provided in Section~\ref{subsec:dataset_oo}.
We learn the representation $\hat\phi_h=\arg\min\limits_{\phi\in\Phi} \hat L_h(\phi)$, where
\begin{align}\label{eq:oa_loss}
	\hat L_h(\phi) &= \frac{1}{T}\sum\limits_{i=1}^{T} \min\limits_{\pi\in\Pi^{\phi}}\max\limits_{g\in\cG} \frac{1}{n}\sum\limits_{j=1}^{n} [K\pi(a^i_j|s^i_j) g(\tilde{s}^i_j) - g(\bar{s}^i_j)]\nonumber\\
	&= \frac{1}{T}\sum\limits_{i=1}^{T} \min\limits_{\pi\in\Pi^{\phi}}\hat{\ell}_h^{\bx^{(i)}}(\pi)
\end{align}
where for dataset $\bx=\{(s_j,a_j,\tilde{s}_j,\bar{s}_j)\}_{j=1}^n$, $\hat{\ell}^\bx_h(\pi)\coloneqq\max\limits_{g\in\cG}\frac{1}{n}\sum\limits_{j=1}^{n} [K\pi(a_j|s_j) g(\tilde{s}_j) - g(\bar{s}_j)]$.
Note that because of the $\max$ operator over the class $\cG$, $\hat{\ell}^\bx_h$ is not an unbiased estimator of $\ell^\mu_h$ when $\bx\sim\mu_h^n$.
However we can still show generalization bounds.

\paragraph{Evaluating representations $\hat{\phi}_1,\dots,\hat{\phi}_H$:}
Learned representations are tested on a new task $\mu\sim\eta$ as follows: get samples $\bx=(\bx_1,\dots,\bx_H)$\footnote{Note that we do not need the datasets $\bx_h$ at different levels to be independent of each other} for all levels using trajectories from $\bpi^*_\mu$, where $\bx_h\sim\mu_h^n$.
For each level $h$, learn $\pi^{\hat{\phi}_h,\bx_h}=\arg\min_{\pi\in\Pi^{\hat{\phi}}}\hat{\ell}_h^{\bx_h}(\pi)$ and consider the policy $\bpi^{\hat{\phi},\bx}=(\pi^{\hat{\phi}_1,\bx_1}, \dots, \pi^{\hat{\phi}_H,\bx_H})$.
Before presenting the guarantee for $\pi^{\hat{\phi},\bx}$, we introduce a notion of {\em Bellman error} that will show up in our results.
For a policy $\bpi=(\pi_1,\dots,\pi_H)$ and an expert policy $\bpi^*=(\pi^*_1,\dots,\pi^*_H)$, we define the inherent Bellman error 
\begin{align}
	\label{eq:bellman_error}
	\epsilon^{\bpi}_{be}\coloneqq\max\limits_{h\in[H]}\max\limits_{g\in\cG}\min\limits_{g'\in\cG}\ex_{s\sim(\nu_h^*+\nu_h^{\bpi})/2}[|g'(s)-(\Gamma^{\bpi}_h g)(s)|]
\end{align}
We make the following two assumptions for the subsequent theorem.
These are standard assumptions in theoretical reinforcement learning literature.
\begin{assumption}[Value function realizability]
	\label{ass:realizability_V_oo}
	$V^*_{h,\mu}\in\gG$ for every $h\in[H]$, $\mu\in\text{support}(\eta)$.
\end{assumption}
\begin{assumption}[Policy realizability]
	\label{ass:realizibility_policy_oo}
	There are representations $\phi^*_1,\dots,\phi^*_H\in\Phi$ such that $\pi^*_{h,\mu}\in\Pi^{\phi^*_h}$  for every $h\in[H]$, $\mu\in\text{support}(\eta)$.
\end{assumption}
Now we present our main theorem for the observation-alone setting.
\begin{theorem}
	\label{thm:meta_oo}
	Let $\hat{\phi}_h\in\arg\min\limits_{\phi\in\Phi}\hat{L}_h(\phi)$. Under Assumptions~\ref{ass:realizability_V_oo},\ref{ass:realizibility_policy_oo}, with probability $1-\delta$ over sampling of $\bX=(\bX_1,\dots,\bX_H)$, we have
\begin{align*}
	\ex_{\mu\sim\eta}\ex_{\bx}&J(\bpi^{\hat{\phi},\bx})-\ex_{\mu\sim\eta}J(\bpi^*_\mu)\le\sum\limits_{h=1}^H(2H-2h+1)\epsilon_{gen,h} + O(H^2)\epsilon^{\hat{\phi}}_{be}
\end{align*}
where $\epsilon^{\hat{\phi}}_{be}=\ex_{\mu\sim\eta}\ex_{\bx}[\epsilon_{be}^{\bpi^{\hat{\phi},\bx}}]$ is the average inherent Bellman error and 
\begin{align*}
	\epsilon_{gen,h}=&c_1\frac{KG(\Phi(\bS_h))}{T\sqrt{n}}+c_2\frac{RK\sqrt{K}}{\sqrt{n}}+c_3\sqrt{\frac{\ln(H/\delta)}{T}} + c_4\ex_{\mu\sim\eta}\ex_{\bx\sim\mu^n}\left[\frac{KG(\cG(\tilde{\bs}_h))}{n}+\frac{G(\cG(\bar{\bs}_h))}{n}\right]
\end{align*}
\end{theorem}
We again give a PAC-style guarantee for the special case where the class of representation functions $\Phi$ and value function class $\cG$ are finite.
It follows from the above theorem and Massart's lemma.
\begin{corollary}
In the setting of Theorem~\ref{thm:meta_oo}, suppose $\Phi,\cG$ are finite. If number of tasks satisfies $T\ge c_1\max\left\{\frac{H^4R^2K^2\log(|\Phi|)}{\epsilon^2}, \frac{H^4\ln(H/\delta)}{\epsilon^2}\right\}$, and number of samples (trajectories) per task satisfies $n\ge c_2\max\left\{\frac{H^4K^2\log(|\cG|)}{\epsilon^2}, \frac{H^4R^2K^3}{\epsilon^2}\right\}$ for small constants $c_1,c_2$, then with probability $1-\delta$,\[
\ex_{\mu\sim\eta}\ex_{\bx}J(\bpi^{\hat{\phi},\bx})-\ex_{\mu\sim\eta}J(\bpi^*_\mu) \le 
O(H^2)\epsilon^{\hat{\phi}}_{be} + \epsilon.
\]
\end{corollary}	
\paragraph{Discussion:} As in the previous section, the number of samples required for a new task after learning a representation is independent of the class $\Phi$ but depends only on the value function class $\cG$ and number of actions.
Thus representation learning is very useful when the class $\Phi$ is much more complicated than $\cG$, i.e. $R^2\log(|\Phi|)\gg \max\{\log(|\cG|),R^2K\}$.
In the above bounds, $\epsilon_{be}^{\hat{\phi}}$ is a Bellman error term.
This type of error terms occur commonly in the analysis of policy iteration type algorithms \citep{munos2005error,munos2008finite}.
We remark that unlike in \citet{sun2019provably}, our Bellman error is based on the Bellman operator of the learned policy rather than the optimal policy.
\citet{le2019batch} used a similar notion that they call {\em inherent Bellman evaluation error}.

The proof of Theorem~\ref{thm:meta_oo} follows a similar outline to that of behavioral cloning.
However we cannot use results from \cite{maurer2016benefit} directly since we are solving a min-max game for each task.
We provide the proof in \Secref{apdx:oo_proofs}.

\section{Experiments}
\label{sec:exp}

\begin{figure*}[t]
	\centering
	\includegraphics[width = .22 \textwidth]{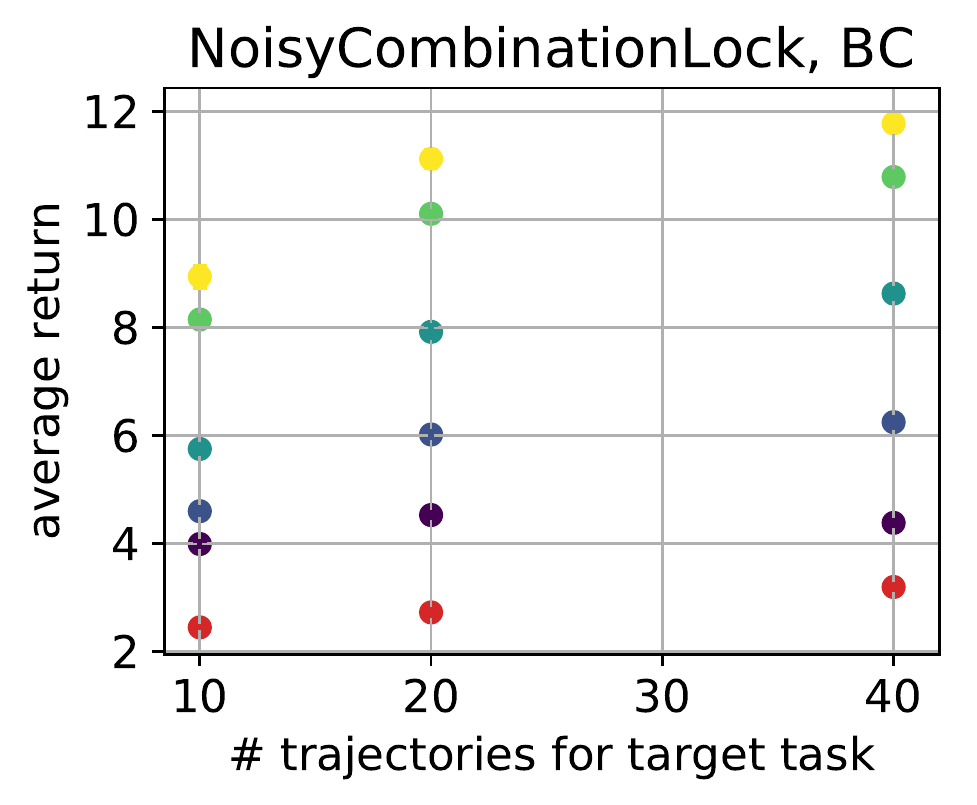}
	\quad
		\includegraphics[width = .22 \textwidth]{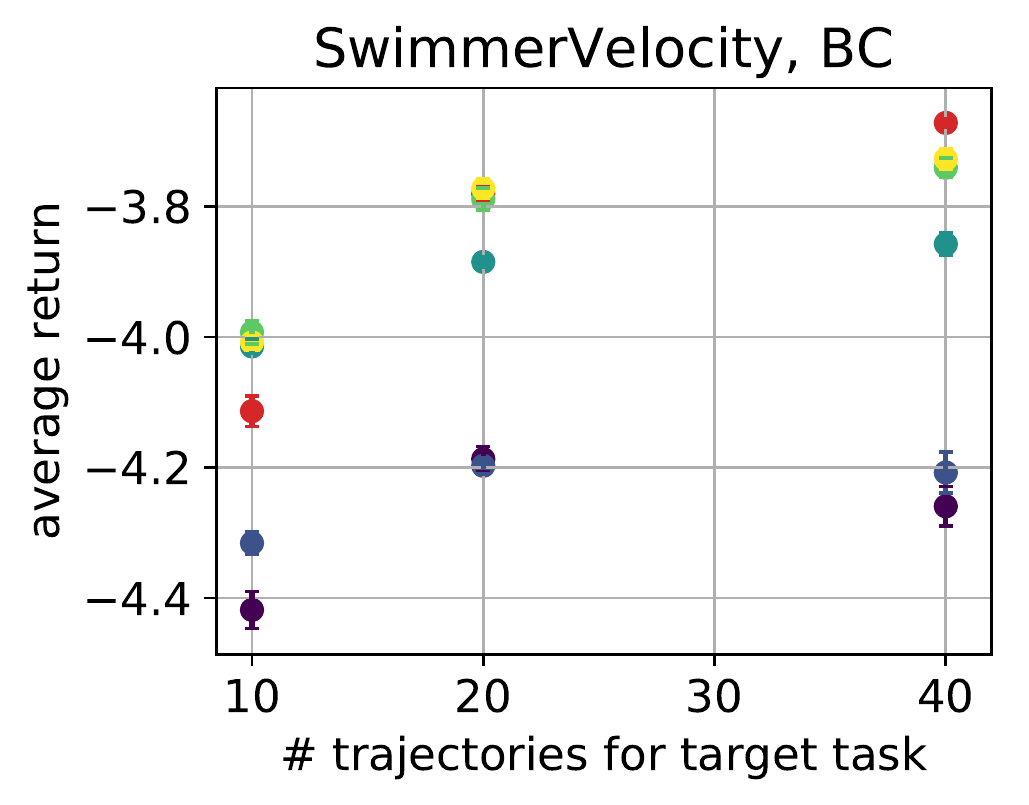}
		\quad
	\includegraphics[width = .22 \textwidth]{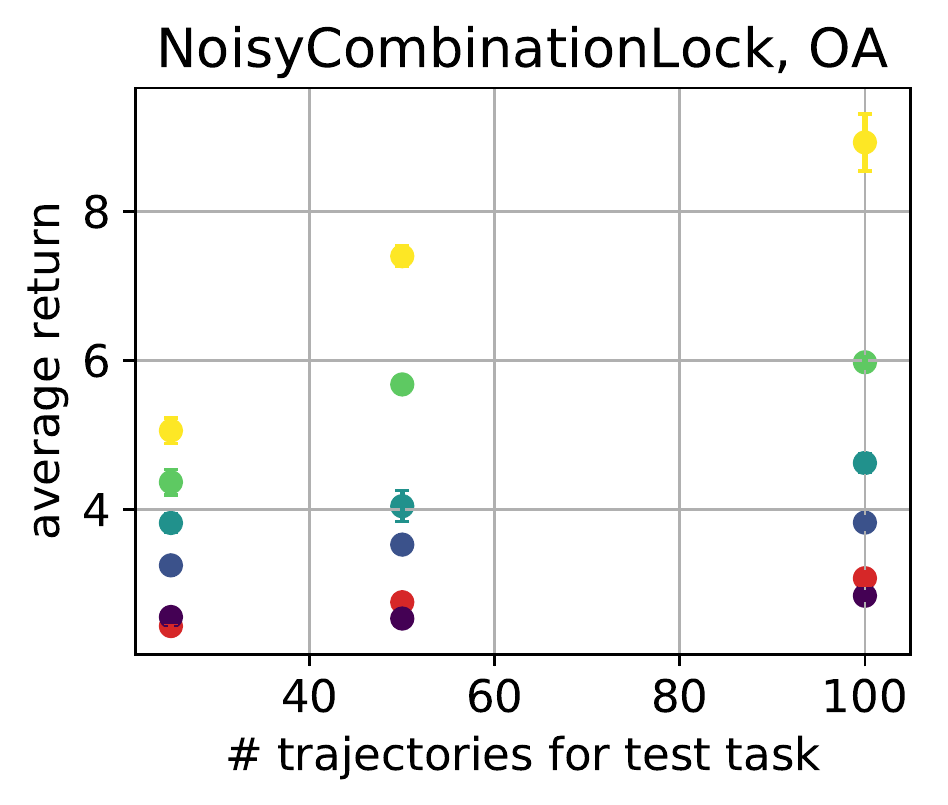}
	\quad
	\includegraphics[width = .22 \textwidth]{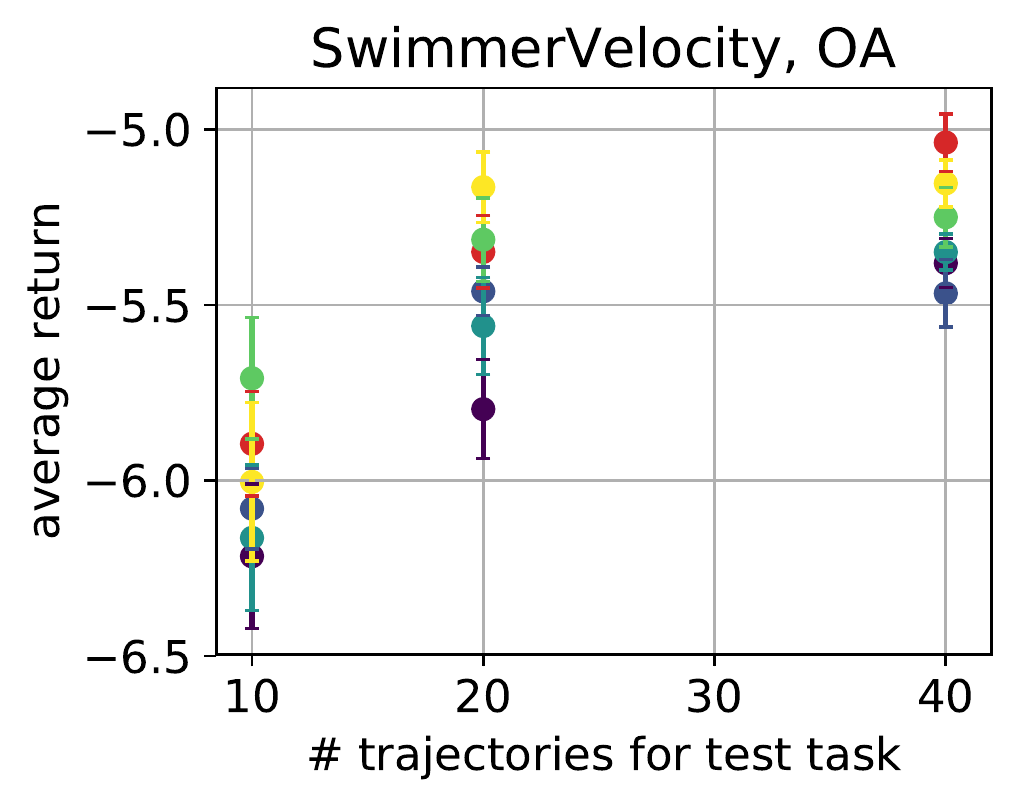}\\
	\includegraphics[width = 0.9 \textwidth]{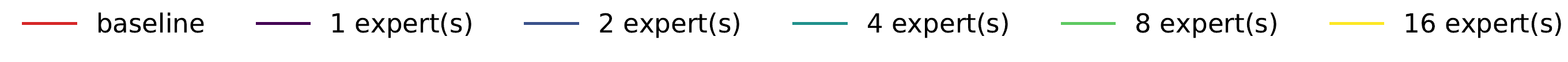}
	\caption{Experiments for verifying our theory. From left to right: Representation learning for behavioral cloning on NoisyCombinationLock,  representation learning for behavioral cloning on SwimmerVelocity,  representation learning for observations alone setting on NoisyCombinationLock,  representation learning for observations alone setting on SwimmerVelocity,
		We compare imitation learning based on learned representation using 1 - 16 experts to the baseline method (without representation learning).
The error bars are calculated using 5 seeds and indicate one standard deviation.
}
	\label{fig:theory}
\end{figure*}

\begin{figure*}[t]
	\centering
		\includegraphics[width = .23 \textwidth]{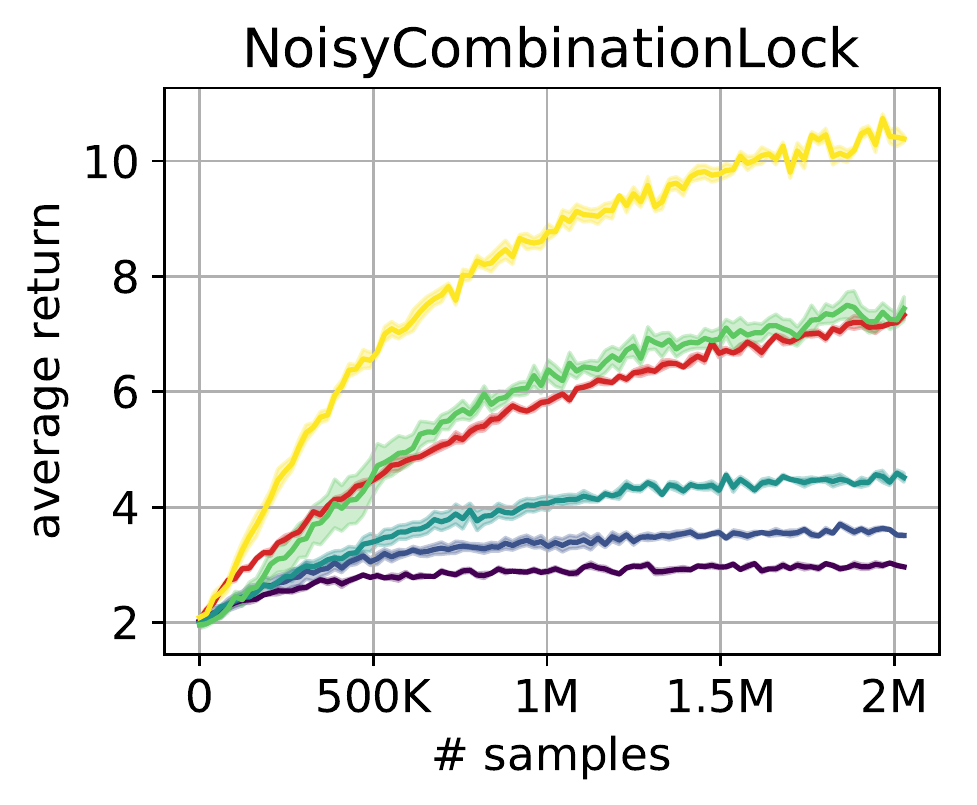}
			\quad
	\includegraphics[width = .23 \textwidth]{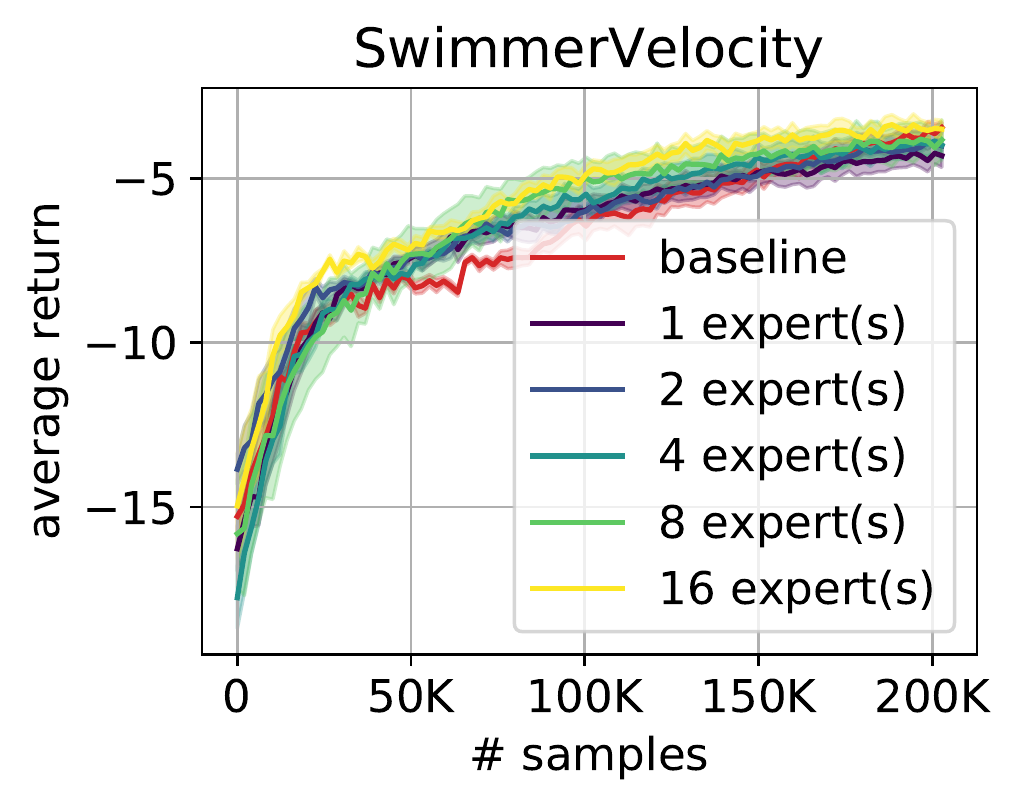}\\
	\includegraphics[width = 0.9 \textwidth]{figures/legend.pdf}
	\caption{Experiments on Policy Optimization with representation trained by imitation learning.
		Left: Results on NoisyCombinationLock. Right: Results on SwimmerVelocity.
We compare policy optimization based on learned representation using 1 - 16 experts to the baseline method (without representation learning).
	}
	\label{fig:rl}
\end{figure*}
In this section we present our experimental results.
These experiments have two aims: \begin{enumerate*}
\item  Verify our theory that representation learning can reduce the sample complexity in the new imitation learning task.
\item Test the power of representations learned via our framework in a broader context.
We wish to see if the learned representation is useful beyond imitation learning and can be used to learn a policy in the RL setting.
\end{enumerate*}
Since our goal of the experiment is to demonstrate the advantage of representation learning, we only consider the standard baseline where for a task we learn a policy $\pi$ from the class $\Pi$ from scratch (without learning a representation first using other tasks).

We conduct our experiments on two environments.
NoisyCombinationLock is a variant of the standard CombinationLock environment \citep{kakade2003sample}, we add additional noisy features to confuse the agent.
Different tasks involve different combinations for the lock.
SwimmerVelocity is a modifed environment the Swimmer environment from OpenAI gym~\citep{Gym}  with Mujoco simulator \citep{mujoco}, and this environment is similar to the one used in \citep{finn2017model}.
The goal in SwimmerVelocity is to move at a target velocity (speed and direction) and the various tasks differ in target velocities.
See Section~\ref{sec:exp_details} for more details about these two environments.

\subsection{Verification of Theory}
We first present our experimental results to verify our theory.
\paragraph{Representation learning for Behavioral Cloning}
We first test our theory on representation learning for behavioral cloning.
We learn the representation using Equation~\ref{eq:L_hat_bc} on the first $T$ tasks.
The specification of policy class and other experiment details are in Section~\ref{sec:exp_details}.

The first plot in Figure~\ref{fig:theory} shows results on the NoisyCombinationLock environment.
We observe that in NoisyCombinationLock, even one expert can help and more experts will always improve the average return.

The second plot in Figure~\ref{fig:theory} shows results on the SwimmerVelocity environment.
Again, more experts always help.
Furthermore, we observe an interesting phenomenon.
When the number of experts is small (2 or 4), the baseline method can outperform policies trained using representation learning, though the baseline method requires more samples to achieve this.
This behavior is actually expected according to our theory.
When the number of experts is small, we may learn a sub-optimal representation and because we fix this representation for training the policy, more samples for the test task cannot make this policy better, whereas more samples always make the baseline method better.

\paragraph{Representation Learning for Observations Alone Setting}
We next verify our theory for the observations alone setting.
We learn the representation using Equation~\ref{eq:oa_loss} on the first $T$ tasks.
Again, the specification of policy class and other experiment details are in Section~\ref{sec:exp_details}.

The results for NoisyCombinationLock and SwimmerVelocity are shown in the third and the fourth plots in Figure~\ref{fig:theory}, respectively.
We observe similar phenomenon as the first and the second plot.
Increasing the number of experts always help and baseline method can outperform policies trained using representation learning when the number of trajectories for the test task is large.

We remark that comparing with the behavioral cloning setting, the observations alone setting often has smaller return.
We suspect the reason is that Equation~\ref{eq:oa_loss} considers the worst case $g$ in $\cG$, thus it prefers pessimistic policies.
Also this setting does not have access to the experts actions as opposed to the behavioral cloning setting.

\subsection{Policy optimization with representations trained by imitation learning}
We test whether the learned representation via imitation learning is useful for the target \emph{reinforcement learning} problem.
We use Equation~\ref{eq:L_hat_bc} to learn representations and we use a proximal policy optimization method~\citep{PPO} to learn a linear policy over the learned representation.
See Section~\ref{sec:exp_details} for details.

The results are reported in Figure~\ref{fig:rl} and are very encouraging.
For both NoisyCombinationLock and SwimmerVelocity environments, we observe that when the number of experts to learn the representation is small, the baseline method enjoys better performance than the policies trained using representation learning.
On the other hand, as the number of experts increases, the policy trained using representation learning can outperform the baseline, sometime significantly.
This experiment suggests that representations trained via imitation learning can be useful \emph{beyond imitation learning}, especially when the target task has few samples.

\section{Conclusion}
\label{sec:con}
The current paper proposes a bi-level optimization framework to formulate and analyze representation learning for imitation learning using multiple demonstrators.
Theoretical guarantees are provided to justify the statistical benefit of representation learning.
Some preliminary experiments verify the effectiveness of the proposed framework.
In particular, in experiments, we find the representation learned via imitation learning is also useful for policy optimization in the reinforcement learning setting.
We believe it is an interesting theoretical question to explain this phenomenon.
Additionally, extending this bi-level optimization framework to incorporate methods beyond imitation learning is an interesting future direction.
Finally, while we fix the learned representation for a new task, once could instead also fine-tune the representation given samples for a new task and a theoretical analysis of this would be of interest.

\bibliography{simonduref}
\bibliographystyle{plainnat}

\newpage
\appendix
\onecolumn
\section{Proofs for Behavioral Cloning}\label{apdx:bc_proofs}
We prove Theorem~\ref{thm:meta_bc} in this section by proving Lemma~\ref{lem:gen_bc},\ref{lem:j_mu_bc}.
In this section, we abuse notation and define $\ell^\mu(\phi,f)\coloneqq\ell^\mu(\pi^{\phi,f})$, where $\ell^\mu$ is defined in Equation~\ref{eq:ell_mu_bc}.
We rewrite it here for convenience.
\begin{align*}
	\ell^\mu(\pi)=\ex_{(s,a)\sim\mu}\ell(\pi(s),a) = \ex_{(s,a)\sim\mu}-\log(\pi(s)_a)
\end{align*}
Let $\hat{f}^{\phi}_\bx=\arg\min\limits_{f\in\gF}\ell^\bx(\phi,f)$ be the optimal task specific parameter for task $\mu$ by fixing representation $\phi$.
Thus by our definitions in Section~\ref{sec:meta_bc}, we get $\pi^{\phi,\bx}=\pi^{\phi,\hat{f}^\phi_\bx}$.
We assume w.l.o.g. that $\actions=[K]$.
Remember that $\ell:\simplex(\actions)\times\actions\rightarrow\mathbb{R}$ is defined as $\ell(\vv,a)=-\log(\vv_a)$ for some $\vv\in\mathbb{R}^K$ and $\vv_a$ is the coordinate corresponding to action $a\in\actions=[K]$.
We define a new function class and loss function that will be useful for our proofs
\begin{align}
	\label{eq:F_prime}
	\gF'=\{x\rightarrow Wx\mid W\in\mathbb{R}^{K\times d},\|W\|_F\le1\}
\end{align}
\begin{align}
	\label{eq:ell_prime}
	\ell'(\vv,a)=-\log(\texttt{softmax}(\vv)_a),\vv\in\mathbb{R}^K,a\in\actions
\end{align}
We basically offloaded the burden of computing $\verb|softmax|$ from the class $\gF$ to the loss $\ell'$.
We can convert any function $f'\in\gF'$ to one in $\gF$ by transforming it to $\verb|softmax|(f')$.
We now proceed to proving the lemmas
\begin{proof}[\bf Proof of Lemma~\ref{lem:gen_bc}]
We can then rewrite the various loss functions from Section~\ref{sec:meta_bc} as follows
\begin{align*}
	\hat{L}(\phi) &= \frac{1}{T}\sum\limits_{i=1}^T\min\limits_{f'\in\gF'}\frac{1}{n}\sum\limits_{j=1}^{n}\ell'(f'(\phi(s)),a)\\
	L(\phi) &= \ex_{\mu\sim\eta}\min\limits_{f'\in\gF'}\ex_{(s,a)\sim\mu}\ell'(f'(\phi(s)),a)\\
	\bar{L}(\phi) &= \ex_{\mu\sim\eta}\ex_{\bx\sim\mu^n}\ex_{(s,a)\sim\mu}\ell'(\hat{f'}^\phi_\bx(\phi(s)),a)
\end{align*}
where $\hat{f'}^\phi_\mu\in\arg\min_{f'\in\gF'}\ell^\bx(\phi,\verb|softmax|(f'))$.
It is easy to show that both $\ell'(\cdot,a)$ $\ell'(f'(\cdot),\cdot)$ are 2-lipschitz in their arguments for every $a\in\actions$ and $f'\in\gF'$.
Using a slightly modified version of Theorem 2(i) from \citet{maurer2016benefit}, we get that for $\hat{\phi}\in\arg\min_{\phi\in\Phi}\hat{L}(\phi)$, with probability at least $1-\delta$ over the choice of $\bX$
\begin{align*}
	\bar{L}(\hat{\phi})-\min\limits_{\phi\in\Phi}L(\phi)&\le \frac{2\sqrt{2\pi}G(\Phi(\bS))}{T\sqrt{n}}+\sqrt{2\pi}Q'\sup_{\phi\in\Phi}\sqrt{\frac{\ex_{\mu\sim\eta,(s,a)\sim\mu}\|\phi(s)\|^2}{n}}+\sqrt{\frac{8\log(4/\delta)}{T}}
\end{align*}
\begin{align}
	\label{eq:maurer_bc}
	\bar{L}(\hat{\phi})-\min\limits_{\phi\in\Phi}L(\phi)\le c\frac{G(\Phi(\bS))}{T\sqrt{n}}+c'\frac{Q'R}{\sqrt{n}}+c''\sqrt{\frac{\log(4/\delta)}{T}}
\end{align}
where $Q'=\sup\limits_{y\in\mathbb{R}^{dn}\setminus\{0\}}\frac{1}{\|y\|}\expect_{}\sup\limits_{f\in\gF'}\sum\limits_{i=1,j=1}^{n,K}\gamma_{ij}f'(y_i)_j$.
First we discuss why we need a modified version of their theorem.
Our setting differs from the setting for Theorem 2 from \citet{maurer2016benefit} in the following ways
\begin{itemize*}
\item $\gF'$ is a class of vector valued function in our case, whereas in \citet{maurer2016benefit} it is assumed to contain scalar valued.
The only place in the proof of the theorem where this shows up is in the definition of $Q'$, which we have updated accordingly.
\item \citet{maurer2016benefit} assumes that $\ell'(\cdot,a)$ is 1-lipschitz for every $a\in\actions$ and that $f'(\cdot)$ is $L$ lipschitz for every $f'\in\gF'$. However the only properties that are used in the proof of Theorem 16 are that $\ell'(\cdot,a)$ is 1-lipschitz and that $\ell'(f'(\cdot),a)$ is $L$-lipschitz for every $a\in\actions$, which is exactly the property that we have. Hence their proof follows through for our setting as well.
\end{itemize*}
\begin{lemma}
	\label{lem:bound_Q}
	$Q'\coloneqq\sup\limits_{y\in\mathbb{R}^{dn}\setminus\{0\}}\frac{1}{\|y\|}\expect_{}\sup\limits_{f\in\gF'}\sum\limits_{i=1,j=1}^{n,K}\gamma_{ij}f'(y_i)_j\le\sqrt{K}$
\end{lemma}
\begin{proof}
\begin{align*}
	Q'&\coloneqq\sup\limits_{y\in\mathbb{R}^{dn}\setminus\{0\}}\frac{1}{\|y\|}\expect_{}\sup\limits_{f\in\gF'}\sum\limits_{i=1,j=1}^{n,K}\gamma_{ij}f'(y_i)_j\\
	&=\sup\limits_{y\in\mathbb{R}^{dn}\setminus\{0\}}\frac{1}{\|y\|}\expect_{}\sup\limits_{\|W\|_F\le1}\sum\limits_{i=1,j=1}^{n,K}\gamma_{ij}\langle W_j,y_i \rangle\\
	&=\sup\limits_{y\in\mathbb{R}^{dn}\setminus\{0\}}\frac{1}{\|y\|}\expect_{}\sup\limits_{\|W\|_F\le1}\sum\limits_{j=1}^{K}\langle W_j,\sum\limits_{i=1}^{n}\gamma_{ij}y_i \rangle\\
	&=^{(a)}\sup\limits_{y\in\mathbb{R}^{dn}\setminus\{0\}}\frac{1}{\|y\|}\expect_{}\sqrt{\sum\limits_{j=1}^{K}\left\|\sum\limits_{i=1}^{n}\gamma_{ij}y_i\right\|^2}\\
	&\le^{(b)}\sup\limits_{y\in\mathbb{R}^{dn}\setminus\{0\}}\frac{1}{\|y\|}\sqrt{\sum\limits_{j=1}^{K}\expect_{}\left\|\sum\limits_{i=1}^{n}\gamma_{ij}y_i\right\|^2}
	= \sup\limits_{y\in\mathbb{R}^{dn}\setminus\{0\}}\frac{1}{\|y\|}\sqrt{\sum\limits_{j=1}^{K}\expect_{}\left[\sum\limits_{i=1}^n\sum\limits_{i'=1}^{n}\gamma_{ij}\gamma_{i'j}\langle y_i, y_{i'}\rangle\right]}\\
	&=^{(c)}\sup\limits_{y\in\mathbb{R}^{dn}\setminus\{0\}}\frac{1}{\|y\|}\sqrt{\sum\limits_{j=1}^{K}\sum\limits_{i=1}^{n}\|y_i\|^2}=\frac{1}{\|y\|}\sqrt{K\|y\|^2}=\sqrt{K}
\end{align*}
where we use Jensen's inequality and linearity of expectation for $(b)$ and properties of standard normal gaussian variables for $(c)$.
For $(a)$ we observe that $\sup_{\|W\|_F\le1}\sum_{j=1}^{K}\langle W_j,A_j\rangle = \sup_{\|W\|_F\le1}\langle W,A\rangle = \|A\|_F = \sum_{j=1}^K \|A_j\|^2$.
\end{proof}
Plugging in Lemma~\ref{lem:bound_Q} into Equation~\ref{eq:maurer_bc} completes the proof.
\end{proof}

We now proceed to prove the next lemma.
\begin{proof}[\bf Proof of Lemma~\ref{lem:j_mu_bc}]
Suppose $\bar{L}(\phi)=\ex_{\mu\sim\eta}\ex_{\bx\sim\mu^n}\ell^\mu(\pi^{\phi,\bx})\le\epsilon$.
Consider a task $\mu\sim\eta$ and samples $\bx\sim\mu^n$ and let $\epsilon_\mu(\bx)=\ell^\mu(\pi^{\phi,\bx})$ so that $\bar{L}(\phi)=\ex_{\mu\sim\eta}\ex_{\bx\sim\mu^n}\epsilon_\mu(\bx)$.
Since $\pi^*_\mu$ is deterministic, we get
\begin{align*}
	\ex_{s\sim\nu^*_\mu}\ex_{a\sim\pi^{\phi,\bx}}\indict\{a\neq\pi^*_\mu(s)\}&=\ex_{s\sim\nu^*_\mu}[1-\pi^{\phi,\bx}(s)_{\pi^*_\mu(s)}]\\
	&\le\ex_{s\sim\nu^*_\mu}[-\log(1-(1-\pi^{\phi,\bx}(s)_{\pi^*_\mu(s)}))]\\
	&=\ex_{s\sim\nu^*_\mu}[-\log(\pi^{\phi,\bx}(s)_{\pi^*_\mu(s)})]=\epsilon_\mu(\bx)
\end{align*}
where we use the fact that $x\le-\log(1-x)$ for $x<1$. for the first inequality.
Thus by using Theorem 2.1 from \citet{ross2011reduction}, we get that $J_\mu(\pi^{\phi,\bx})-J_\mu(\pi^*)\le H^2\epsilon_\mu(\bx)$.
Taking expectation w.r.t. $\mu\sim\eta$ and $\bx\sim\mu^n$ completes the proof.
\end{proof}

\begin{proof}[\bf Proof of Theorem~\ref{thm:meta_bc}]
By using Assumption~\ref{ass:realizability_policy}, we are guaranteed the existence of $\pi_\mu\in\Pi^{\phi^*}$ such that $\pi_\mu(s)_{\phi^*_\mu(s)}\ge1-\gamma$ for every $s\in\states$.
Thus we can get an upper bound on $L(\phi^)$
\begin{align*}
	L(\phi^*)&=\ex_{\mu\sim\eta}\min\limits_{\pi\in\Pi^{\phi^*}}\ex_{s\sim\nu^*_\mu}-\log(\pi(s)_{\pi^*_\mu(s)})\\
	&\le\ex_{\mu\sim\eta}\ex_{s\sim\nu^*_\mu}-\log(\pi_\mu(s)_{\pi^*_\mu(s)})\\
	&\le\ex_{\mu\sim\eta}\ex_{s\sim\nu^*_\mu}-\log(1-\gamma)\le2\gamma
\end{align*}
where in the last step we used $-\log(1-x)\le 2x$ for $x<1/2$.
Hence from Lemma~\ref{lem:gen_bc} we get $\bar{L}(\hat{\phi})\le 2\gamma+\epsilon_{gen,h}$, which combining with Lemma~\ref{lem:j_mu_bc} gives the desired result.
\end{proof}

\section{Proofs for Observation-Alone}\label{apdx:oo_proofs}
Before proving Theorem~\ref{thm:meta_oo}, we introduce the following loss functions, as we did in the proof sketch for the behavioral cloning setting.
We again abuse notation and define $\ell^\mu(\phi,f)\coloneqq\ell^\mu(\pi^{\phi,f})$, where $\ell^\mu$ is defined in Equation~\ref{eq:ell_mu_oo}.
Let $\hat{f}^{\phi}_\bx=\arg\min\limits_{f\in\gF}\ell^\bx(\phi,f)$ be the optimal task specific parameter for task $\mu$ by fixing representation $\phi$.
As before, we define the following
\begin{align*}
\bar{L}_h(\phi_h)=\ex_{\mu\sim\eta}\ex_{\bx\sim\mu_h^n}\ell^{\mu}_h(\phi,\hat{f}_\bx^{\phi_h})
\end{align*}
We first show a guarantee on the performance of representations $(\hat{\phi}_1,\dots,\hat{\phi}_H)$ as measured by the functions $\bar{L}_1,\dots,\bar{L}_H$.
\begin{theorem}\label{thm:gen_oo}
With probability at least $1-\delta$ in the draw of $\bX=(\bX^{(1)},\dots,\bX^{(H)})$, $\forall h\in[H]$
\begin{align*}
\bar{L}_h(\hat{\phi}_h)\le \min\limits_{\phi\in\Phi}L_h(\phi)+c\epsilon_{gen,h}(\Phi)+c'\epsilon_{gen,h}(\cF,\cG)+ c''\sqrt{\frac{\ln(H/\delta)}{T}}
\end{align*}
where $\epsilon_{gen,h}(\Phi)=\frac{KG(\Phi(\bS_h))}{T\sqrt{n}}$ and $\epsilon_{gen,h}(\cF,\cG)=\ex_{\mu\sim\eta}\ex_{\bx\sim\mu^n}\left[\frac{KG(\cG(\tilde{\bs}_h))}{n}+\frac{G(\cG(\bar{\bs}_h))}{n}\right]+\frac{RK\sqrt{K}}{\sqrt{n}}$
\end{theorem}
We then connect the losses $\bar{L}_h$ to the expected cost on the tasks.

\begin{theorem}\label{thm:obs_cost}
Consider representations $(\phi_1,\dots,\phi_H)$ with $\bar{L}_h(\phi_h)\le\epsilon_h$. Let $\bx=(\bx_1,\dots,\bx_H)$ be samples at different levels for a newly sampled task $\mu\sim\eta$ such that $\bx_h\sim\mu_h^n$. Let $\bpi^{\phi,\bx}=(\pi^{\phi_1,\bx_1},\dots,\pi^{\phi_H,\bx_H})$ be policies learned using the samples, then under Assumption~\ref{ass:realizability_V_oo},
\begin{align*}
\ex_{\mu\sim\eta}\ex_{\bx}J(\bpi^{\phi,\bx})-\ex_{\mu\sim\eta}J(\bpi^*_\mu)\le\sum\limits_{h=1}^{H}(2H-2h+1)\epsilon_h + O(H^2)\epsilon^{\phi}_{be}
\end{align*}
where $\epsilon^{\phi}_{be}=\ex_{\mu\sim\eta}\ex_{\bx}[\epsilon_{be}^{\bpi^{\phi,\bx}}]$ is the average inherent Bellman error.
\end{theorem}

It is easy to show that under Assumption~\ref{ass:realizibility_policy_oo}, $\min_{\phi\in\Phi}L_h(\phi)=0$ for every $h\in[H]$.
Thus from Theorem~\ref{thm:gen_oo}, we get that $\bar{L}_h(\hat{\phi}_h)\le\epsilon_{gen,h}$, where $\epsilon_{gen,h}=\epsilon_{gen,h}(\Phi)+\epsilon_{gen,h}(\gF,\gG)+ c''\sqrt{\frac{\ln(H/\delta)}{T}}$.
Invoking Theorem~\ref{thm:obs_cost} on the representations $\{\hat{\phi}_h\}$ completes the proof.

\subsection{Proof of Theorem~\ref{thm:gen_oo}}\label{apdx:b1}
Before proving the theorem, we discuss important lemmas.
In yet another abuse of notation, we define $\ell^\mu_h(\phi,f,g)=\expect_{(s,a,\tilde{s},\bar{s})\sim\mu_h}[K\pi^{\phi,f}(a|s) g(\tilde{s}) - g(\bar{s})]$ and $\ell^\bx_h(\phi,f,g)=\frac{1}{n}\sum\limits_{j=1}^{n} [K\pi^{\phi,f}(a_j|s_j) g(\tilde{s}_j) - g(\bar{s}_j)]$.

Let $\hat{m}_{\bx}(\phi)=\min\limits_{f\in\cF}\max\limits_{g\in\cG}\hat{\ell}^{\bx}_h(\phi,f,g)=\hat{\ell}^{\bx}_h(\phi,\hat{f}^{\phi}_{\bx},\hat{g}^{\phi}_{\bx})$, $\bar{m}_{\mu,\bx}(\phi)=\max\limits_{g\in\cG}\ell^{\mu}_h(\phi,\hat{f}^{\phi}_\bx,g)$ and $m_\mu(\phi)=\min\limits_{f\in\cF}\max\limits_{g\in\cG}\ell^\mu_h(\phi,f,g)$.
Note that $L_h(\phi)=\ex_{\mu\sim\eta}m(\phi)$, $\bar{L}_h(\phi)=\ex_{\mu\sim\eta}\ex_{\bx\sim\mu^n}\bar{m}_{\mu,\bx}(\phi)$.
Define the distribution $\rho_h$ where $\bx\sim\rho_h$ is the same as $\mu\sim\eta$ and then $\bx\sim\mu_h^n$.
\begin{lemma}\label{lem:gen1}
For every $\phi\in\Phi$ and $h\in[H]$,
\begin{align*}
\ex_{\mu\sim\eta}\ex_{\bx\sim\mu^n}\sup\limits_{f\in\cF}\sup\limits_{g\in\cG}\left[\hat{\ell}^\bx_h(\phi,f,g)-\ell^\mu_h(\phi,f,g)\right]\le\epsilon_{gen,h}(\cF,\cG)
\end{align*}
\end{lemma}
\begin{lemma}\label{lem:gen2}
With probability $1-\delta$, for every $\phi\in\Phi$,
\begin{align*}
\bar{L}_h(\phi)-\ex_{\bx\sim\rho_h}\hat{m}_\bx(\phi)\le\epsilon_{gen,h}(\cF,\cG)
\end{align*}
\end{lemma}
\begin{lemma}\label{lem:gen3}
With probability $1-\delta$, for every $\phi\in\Phi$,
\begin{align*}
\ex_{\bx\sim\rho_h}\hat{m}_\bx(\phi)-\frac{1}{T}\sum\limits_{i}\hat{m}_{\bx^{(i)}}(\phi)\le \epsilon_{gen,h}(\Phi)+O\left(\sqrt{\frac{\log(\frac{1}{\delta})}{T}}\right)
\end{align*}
\end{lemma}
\noindent We prove these lemmas later. First we prove Theorem~\ref{thm:gen_oo} using them.
If $\phi^*_h=\arg\min\limits_{\phi\in\Phi} L_h(\phi)$, then
\begin{align*}
\bar{L}_h(\hat{\phi}_h)-L_h(\phi^*_h)=&\left(\bar{L}_h(\hat{\phi}_h)-\ex_{\bx\sim\rho_h}\hat{m}_\bx(\phi)\right)\\
&+\left(\ex_{\bx\sim\rho_h}\hat{m}_\bx(\phi)-\frac{1}{T}\sum\limits_i\hat{m}_{\bx^{(i)}}(\hat{\phi}_h)\right)\\
&+\left(\frac{1}{T}\sum\limits_i\hat{m}_{\bx^{(i)}}(\hat{\phi}_h)-\frac{1}{T}\sum\limits_i\hat{m}_{\bx^{(i)}}(\phi^*_h)\right)\\
&+\left(\frac{1}{T}\sum\limits_i\hat{m}_{\bx^{(i)}}(\phi^*_h)-\ex_{\bx\sim\rho_h}\hat{m}_{\bx}(\phi^*_h)\right)\\
&+\ex_{\mu\sim\eta}[\ex_{\bx\sim\mu^n}\hat{m}_{\bx}(\phi^*_h)-m_\mu(\phi^*_h)]\\
&\le2\epsilon_{gen,h}(\cF,\cG)+\epsilon_{gen,h}(\Phi)+O\left(\sqrt{\frac{\log(\frac{1}{\delta})}{T}}\right)
\end{align*}
where for the first part we use Lemma~\ref{lem:gen2}, second part we use Lemma~\ref{lem:gen3}, third part is upper bounded by 0 by optimality of $\hat{\phi}_h$, fourth is upper bounded by $O(\sqrt{\frac{\log(\frac{1}{\delta})}{T}})$ by Hoeffding's inequality and fifth is bounded by the following argument: let $f^{\phi},g^{\phi}=\arg\min\limits_{f\in\cF}\arg\max\limits_{g\in\cG}\ell^\mu(\phi,f,g)$
\begin{align*}
\ex_{\bx\sim\mu^n}\hat{m}_{\bx}(\phi^*_h)&=\ex_{\bx\sim\mu^n}\min\limits_{f\in\cF}\max\limits_{g\in\cG}\hat{\ell}^{\bx}_h(\phi^*_h,f,g)\\&\le\ex_{\bx\sim\mu^n}\max\limits_{g\in\cG}\hat{\ell}^{\bx}_h(\phi^*_h,f^{\phi^*_h},g)\\
&=^{(a)}\ex_{\bx\sim\mu^n}\hat{\ell}^{\bx}_h(\phi^*_h,f^{\phi^*_h},\tilde{g})\\
&\le^{(b)}\ell^{\mu}_h(\phi^*_h,f^{\phi^*_h},\tilde{g})+\epsilon_{gen,h}(\cF,\cG)\\&\le\ell^{\mu}_h(\phi^*_h,f^{\phi^*_h},g^{\phi^*_h})+\epsilon_{gen,h}(\cF,\cG)=m_\mu(\phi^*_h)+\epsilon_{gen,h}(\cF,\cG)
\end{align*}
where in step $(a)$ we use $\tilde{g}=\arg\max\limits_{g\in\cG}\hat{\ell}^{\bx}_h(\phi^*_h,f^{\phi^*_h},g)$, for $(b)$ we use Lemma~\ref{lem:gen1}.

\subsection{Proof of Theorem~\ref{thm:obs_cost}}
Consider a task $\mu$.
For simplicity of notation, we use $\pi_h$ instead $\pi^{\phi_h,\bx_h}$, $\bpi$ instead of $\bpi^{\phi,\bx}$.
Let $\nu_h^\pi$ and $\nu_h^*$ be the state distributions at level $h$ induced by $\bpi^{\phi,\bx}$ and $\bpi^*_\mu$ respectively.
Let
\begin{align*}
\epsilon_h(\bx_h)=\max\limits_{g\in\cG}\ex_{s\sim\nu_h^*}[\ex_{\substack{a\sim\pi_h\\s'\sim P_{s,a}}}g(s')-\ex_{\substack{a\sim\pi_h^*\\s'\sim P_{s,a}}}g(s')]
\end{align*}
be the loss of policy $\pi_h$ at level $h$. By definition, $\epsilon_h=\ex_{\mu\sim\eta}\ex_{\bx\sim\mu_h^n} \epsilon_h(\bx)$.
Using Lemma C.1 from \cite{sun2019provably}, we have
\begin{align*}
J(\bpi^{\phi,\bx})-J(\bpi^*_\mu)&=\sum\limits_{h=1}^H\bar{\Delta}_h=\sum\limits_{h=1}^H\ep_{s\sim\nu_h^\pi}\left[\ex_{\substack{a\sim\pi_h(\cdot|s),s'\sim P_{s,a}}}V^*_{h+1}(s')-\ex_{\substack{a\sim\pi^*_h(\cdot|s),s'\sim P_{s,a}}}V^*_{h+1}(s')\right]
\end{align*}
Observe that 
\begin{align*}
\bar{\Delta}_h&=&&\ex_{s\sim\nu_h^\pi}[\ex_{\substack{a\sim\pi_h(\cdot|s),s'\sim P_{s,a}}}V^*_{h+1}(s')-\ex_{\substack{a\sim\pi^*_h(\cdot|s),s'\sim P_{s,a}}}V^*_{h+1}(s')]\\
&\le^{(a)}&&\ex_{s\sim\nu_h^*}\ex_{\substack{a\sim\pi_h(\cdot|s),s'\sim P_{s,a}}}V^*_{h+1}(s')-\ex_{s\sim\nu_h^*}\ex_{\substack{a\sim\pi^*_h(\cdot|s),s'\sim P_{s,a}}}V^*_{h+1}(s')+\\
&&&\ex_{s\sim\nu_h^\pi}\ex_{\substack{a\sim\pi_h(\cdot|s),s'\sim P_{s,a}}}V^*_{h+1}(s)-\ex_{s\sim\nu_h^*}\ex_{\substack{a\sim\pi_h(\cdot|s),s'\sim P_{s,a}}}V^*_{h+1}(s)+\\
&&&\ex_{s\sim\nu_h^*}\ex_{\substack{a\sim\pi^*_h(\cdot|s),s'\sim P_{s,a}}}V^*_{h+1}(s')-\ex_{s\sim\nu_h^\pi}\ex_{\substack{a\sim\pi^*_h(\cdot|s),s'\sim P_{s,a}}}V^*_{h+1}(s')\\
&\le^{(b)}&&\max\limits_{g\in\cG}\ex_{s\sim\nu_h^*}[\ex_{\substack{a\sim\pi_h(\cdot|s),s'\sim P_{s,a}}}g(s')-\ex_{\substack{a\sim\pi^*_h(\cdot|s),s'\sim P_{s,a}}}g(s')]+\\
&&&\max\limits_{g\in\cG}[\ex_{s\sim\nu_h^\pi}\Gamma_h^{\bpi}g(s)-\ex_{s\sim\nu_h^*}\Gamma_h^{\bpi}g(s)]+[\ex_{s\sim\nu_h^*}\Gamma_h^{*}V^*_{h+1}(s)-\ex_{s\sim\nu_h^\pi}\Gamma_h^{*}V^*_{h+1}(s)]\\
&\le^{(c)}&&\epsilon_h(\bx_h)+\max\limits_{g\in\cG}[\ex_{s\sim\nu_h^\pi}\Gamma_h^{\bpi}g(s)-\ex_{s\sim\nu_h^*}\Gamma_h^{\bpi}g(s)]+\max\limits_{g\in\cG}[\ex_{s\sim\nu_h^\pi}g(s)-\ex_{s\sim\nu_h^*}g(s)]
\end{align*}
where $(a)$ just adds and subtracts terms, $(b)$ uses the assumption that $V^*_{h+1}\in\cG$ and the definitions of $\Gamma_h^*$ and $\Gamma_h^{\bpi}$ from \Secref{sec:pre} and $(c)$ uses the definition of $\epsilon_h(\bx_h)$.
The following lemma helps us bound the remaining two terms.
\begin{lemma}\label{lem:delta}
Defining $\Delta_h=\max\limits_{g\in\cG}|\ex_{s\sim\nu_h^\pi}g(s)-\ex_{s\sim\nu_h^*}g(s)|$, we have
\begin{align*}
\max\limits_{g\in\cG}[\ex_{s\sim\nu_h^\pi}\Gamma_h^{\bpi}g(s)-\ex_{s\sim\nu_h^*}\Gamma_h^{\bpi}g(s)]\le\Delta_h+2\epsilon_{be}^\bpi
\end{align*}
\end{lemma}
\noindent Using the above lemma, we get $\bar{\Delta}_h\le\epsilon_h(\bx_h)+2\Delta_h+2\epsilon_{be}^\bpi$.
We now bound $\Delta_h$
\begin{align*}
\Delta_h&=\max\limits_{g\in\cG}\left|\ex_{s\sim\nu_{h-1}^\pi}\ex_{\substack{a\sim\pi_{h-1}\\s'\sim P_{s,a}}}g(s')-\ex_{s\sim\nu_{h}^*}g(s)\right|\\
&\le^{(a)}\max\limits_{g\in\cG}\left|\ex_{s\sim\nu_{h-1}^\pi}\ex_{\substack{a\sim\pi_{h-1}\\s'\sim P_{s,a}}}g(s')-\ex_{s\sim\nu_{h-1}^*}\ex_{\substack{a\sim\pi_{h-1}\\s'\sim P_{s,a}}}g(s)\right|+\max\limits_{g\in\cG}\left|\ex_{s\sim\nu_{h-1}^*}\ex_{\substack{a\sim\pi_{h-1}\\s'\sim P_{s,a}}}g(s')-\ex_{s\sim\nu_{h}^*}g(s)\right|\\
&=\max\limits_{g\in\cG}\left|\ex_{s\sim\nu_{h-1}^\pi}\Gamma^{\bpi}_{h-1}g(s')-\ex_{s\sim\nu_{h-1}^*}\Gamma^{\bpi}_{h-1}g(s)\right|+\epsilon_{h-1}(\bx_{h-1})\\
&\le\Delta_{h-1}+2\epsilon_{be}^\bpi+\epsilon_{h-1}(\bx_{h-1})
\end{align*}
where $(a)$ uses triangle inequality.
Thus $\Delta_h\le2(h-1)\epsilon_{be}^\bpi+\epsilon_{1:h-1}(\bx_{1:h-1})$ and so $\bar{\Delta}_h\le\epsilon_{1:h}(\bx_{1:h})+\epsilon_{1:h-1}(\bx_{1:h-1})+(4h-2)\epsilon_{be}^\bpi$.
This implies that
\begin{align*}
J(\bpi^{\phi,\bx})-J(\bpi^*)=\sum\limits_{h=1}^H\bar{\Delta}_h\le\sum\limits_{h=1}^H(2H-2h+1)\epsilon_h(\bx_h)+O(H^2)\epsilon_{be}^{\bpi^{\phi,\bx}}\end{align*}
Taking expectation wrt $\mu\sim\eta$ and $\bx\sim\mu^n$ completes the proof.

\subsection{Proofs of Lemmas}
In the following proofs, we will require the well known Slepian's lemma which lets us exploit lipschitzness of functions in gaussian averages
\begin{lemma}[Slepian's lemma]
Let $\{X\}_{s\in S}$ and $\{Y\}_{s\in S}$ be zero mean Gaussian processes such that
\begin{align*}
\expect(X_s-X_t)^2\le\expect(Y_s-Y_t)^2, \forall s,t\in S
\end{align*}
Then
\begin{align*}
\expect\sup\limits_{s\in S}X_s\le\expect\sup\limits_{s\in S}Y_s
\end{align*}
\end{lemma}
We now move on to proving earlier lemmas.
\begin{proof}[\bf Proof of Lemma~\ref{lem:gen1}]
Again we define $\gF'$ as in Equation~\ref{eq:F_prime}.
Let $\ell(\vv,\alpha,\beta,a)=K\texttt{softmax}(\vv)_a\alpha-\beta$, and let $\ell'^\mu_h(\phi,f',g)=\ell^\mu_h(\phi,\texttt{softmax}(f'),g)=\ex_{(s,a,\tilde{s},\bar{s})\sim\mu_h}\ell(f'(\phi(s)),g(\tilde{s}),g(\bar{s}),a)$ for $f'\in\gF'$ and similarly define $\hat{\ell'}^\bx_h(\phi,f',g)=\hat{\ell}^\bx_h(\phi,\texttt{softmax}(f'),g)$.
Notice that $\ell(\cdot,\alpha,\beta,a)$ is $2K$-lipschitz, $\ell(\vv,\cdot,\beta,a)$ is $K$-lipschitz and $\ell(\vv,\alpha,\cdot,a)$ is $1$-lipschitz, 
Using Theorem 8(i) from \cite{maurer2016benefit}, we get that
\begin{align*}
\ex_{\mu\sim\eta}\ex_{\bx\sim\mu^n}&\sup\limits_{f\in\cF}\sup\limits_{g\in\cG}\left[\hat{\ell}^\bx_h(\phi,f,g)-\ell^\mu_h(\phi,f,g)\right]\\
&=\ex_{\mu\sim\eta}\ex_{\bx\sim\mu^n}\sup\limits_{f'\in\cF'}\sup\limits_{g\in\cG}\left[\hat{\ell'}^\bx_h(\phi,f',g)-\ell'^\mu_h(\phi,f',g)\right]\\
&\le\frac{\sqrt{2\pi}\ep_{\bx}G(\ell(\cF'(\phi(\bs_h)), \cG(\bar{\bs}_h),\cG(\tilde{\bs}_h), \ba))}{n}
\end{align*}
where the gaussian average is defined as 
\begin{align*}
	G(\ell(\cF'(\phi(\bs_h)), \cG(\bar{\bs}_h),\cG(\tilde{\bs}_h), \ba)) = \ex_{\gamma_i}\left[\sup_{f'\in\cF,g\in\cG} \sum\limits_{i=1}^n\gamma_i\ell(f'(\phi(s_i)),g(\tilde{s}_i), g(\bar{s}_i), a_i)\right]
\end{align*}
where $\bs_h=\{s_i\}_{i=1}^n$, $\bar{\bs}_h=\{\bar{s}_i\}_{i=1}^n$, $\tilde{\bs}_h=\{\tilde{s}_i\}_{i=1}^n$.
We will now use the lipschitzness of $\ell$ to get the following.
\begin{claim}\label{cl:lips}
\begin{align*}
	(\ell(f'_1(\phi(s_i)),g_1(\tilde{s}_i), g_1(\bar{s}_i), a_i) - \ell(f'_2(\phi(s_i)),g_2(\tilde{s}_i), g_2(\bar{s}_i), a_i))^2 &\le 12K^2\|f'_1(\phi(s))-f'_2(\phi(s))\|^2\\
	&+3K^2(g_1(\phi(\tilde{s}))-g_2(\phi(\tilde{s})))^2\\
	&+3(g_1(\phi(\bar{s}))-g_2(\phi(\bar{s})))^2
\end{align*}
\end{claim}
This follows by writing 
\begin{align*}
	\ell(f'_1(\phi(s_i)),g_1(\tilde{s}_i), g_1(\bar{s}_i), a_i) -& \ell(f'_2(\phi(s_i)),g_2(\tilde{s}_i), g_2(\bar{s}_i), a_i) = \\
	&\ell(f'_1(\phi(s_i)),g_1(\tilde{s}_i), g_1(\bar{s}_i), a_i) - \ell(f'_2(\phi(s_i)),g_1(\tilde{s}_i), g_1(\bar{s}_i), a_i) + \\
	&\ell(f'_2(\phi(s_i)),g_1(\tilde{s}_i), g_1(\bar{s}_i), a_i) - \ell(f'_2(\phi(s_i)),g_2(\tilde{s}_i), g_1(\bar{s}_i), a_i) + \\
	&\ell(f'_2(\phi(s_i)),g_2(\tilde{s}_i), g_1(\bar{s}_i), a_i) - \ell(f'_2(\phi(s_i)),g_2(\tilde{s}_i), g_2(\bar{s}_i), a_i)
\end{align*}
and then using the per argument lipschitzness of $\ell$ described earlier and AM-RMS inequality proves the claim.
We move on to decoupling the gaussian average using Slepian's lemma
\begin{claim}
The gaussian average satisfies the following
\begin{align*}
	G(\ell(\cF'(\phi(\bs_h)), \cG(\bar{\bs}_h),\cG(\tilde{\bs}_h), \ba) \le 2\sqrt{3}KG(\cF'(\phi(\bs_h))) + \sqrt{3}KG(\gG(\tilde{\bs}_h)) + \sqrt{3}G(\gG(\bar{\bs}_h))
\end{align*}
where the gaussian average for a class of functions is defined in \Eqref{eq:gaussian_avg}.
\end{claim}
This can be shown by defining two gaussian processes $X_{f',g}=\sum\limits_{i=1}^n\gamma_i\ell(f'(\phi(s_i)),g(\tilde{s}_i), g(\bar{s}_i), a_i)$ and $Y_{f',g}=\sum\limits_{i=1,j=1}^{n,d} \alpha_{i,j} 2\sqrt{3}Kf'(\phi(s_i))_j + \sum\limits_{i=1}^n \beta_i \sqrt{3}Kg(\tilde{s}_i) + \sum\limits_{i=1}^n \delta_i \sqrt{3}g(\tilde{s}_i)$.
It is easy to see the following using expectation of independent gaussian variables
\begin{align*}
	\ex_{\gamma} (X_{f'_1,g_1}-X_{f'_2,g_2})^2 &= \sum_{i=1}^n (\ell(f'_1(\phi(s_i)),g_1(\tilde{s}_i), g_1(\bar{s}_i), a_i) -\ell(f'_2(\phi(s_i)),g_2(\tilde{s}_i), g_2(\bar{s}_i), a_i))^2\\
	\ex_{\alpha,\beta,\delta} (Y_{f'_1,g_1}-Y_{f'_2,g_2})^2 &= 12K^2\|f'_1(\phi(s))-f'_2(\phi(s))\|^2 + 3K^2(g_1(\phi(\tilde{s}))-g_2(\phi(\tilde{s})))^2 + (g_1(\phi(\bar{s}))-g_2(\phi(\bar{s})))^2
\end{align*}
Claim~\ref{cl:lips} gives us that $\ex_{\gamma} (X_{f'_1,g_1}-X_{f'_2,g_2})^2 \le \ex_{\alpha,\beta,\delta} (Y_{f'_1,g_1}-Y_{f'_2,g_2})^2$ and then Slepian's lemma will then give us that 
\begin{align*}
	\ex_{\gamma}\sup\limits_{f',g}\sum\limits_{i=1}^n\gamma_i&\ell(f'(\phi(s_i)),g(\tilde{s}_i), g(\bar{s}_i), a_i) \\
	&\le \ex_{\alpha,\beta,\delta}\sup\limits_{f',g} \left[\sum\limits_{i=1,j=1}^{n,d} \alpha_{i,j} 2\sqrt{3}Kf'(\phi(s_i))_j + \sum\limits_{i=1}^n \beta_i \sqrt{3}Kg(\tilde{s}_i) + \sum\limits_{i=1}^n \delta_i \sqrt{3}g(\tilde{s}_i)\right]\\
	&\le \ex_{\alpha}\sup\limits_{f'} \left[\sum\limits_{i=1,j=1}^{n,d} \alpha_{i,j} 2\sqrt{3}Kf'(\phi(s_i))_j\right] + \ex_{\beta}\sup\limits_{g} \left[\sum\limits_{i=1}^n \beta_i \sqrt{3}Kg(\tilde{s}_i)\right] + \ex_{\delta}\sup\limits_{g} \left[\sum\limits_{i=1}^n \delta_i \sqrt{3}g(\tilde{s}_i)\right]\\
	& = 2\sqrt{3}KG(\cF'(\phi(\bs_h))) + \sqrt{3}KG(\gG(\tilde{\bs}_h)) + \sqrt{3}G(\gG(\bar{\bs}_h))
\end{align*}
thus proving the claim.
Furthermore, we notice that $G(\cF'(\phi(\bs_h)))\le Q'$, where $Q'$ is defined in Lemma~\ref{lem:bound_Q}.
Thus combining all of this, we get 
\begin{align*}
	\ex_{\mu\sim\eta}\ex_{\bx\sim\mu^n}&\sup\limits_{f\in\cF}\sup\limits_{g\in\cG}\left[\hat{\ell}^\bx_h(\phi,f,g)-\ell^\mu_h(\phi,f,g)\right]\\
	&\le\ex_{\mu\sim\eta}\ex_{\bx\sim\mu^n}\frac{2\sqrt{6\pi}KG(\cF'(\phi(\bs_h)))}{n}+\ex_{\mu\sim\eta}\ex_{\bx\sim\mu^n}\left[\frac{\sqrt{6\pi}KG(\cG(\bar{\bs}_h))}{n}+\frac{\sqrt{6\pi}G(\cG(\tilde{\bs}_h))}{n}\right]\\
	&\le\frac{2\sqrt{6\pi}KQ'}{n}+\ex_{\mu\sim\eta}\ex_{\bx\sim\mu^n}\left[\frac{\sqrt{6\pi}KG(\cG(\bar{\bs}_h))}{n}+\frac{\sqrt{6\pi}G(\cG(\tilde{\bs}_h))}{n}\right]\\
	&\le c\frac{RK\sqrt{K}}{\sqrt{n}}+c'\ex_{\mu\sim\eta}\ex_{\bx\sim\mu^n}\left[\frac{KG(\gG(\bar{\bs}_h))}{n}+\frac{G(\gG(\tilde{\bs}_h))}{n}\right]\le \epsilon_{gen,h}(\cF,\cG)
\end{align*}
where we used Lemma~\ref{lem:bound_Q} for the last inequality.
This completes the proof
\end{proof}
\begin{proof}[\bf Proof of Lemma~\ref{lem:gen2}]
\begin{align*}
\bar{L}_h(\phi)-\ex_{\bx\sim\rho_h}\hat{m}_\bx(\phi)&=\ex_{\mu\sim\eta}\ex_{\bx\sim\mu^n}\bar{m}_{\mu,\bx}(\phi)-\ex_{\mu\sim\eta}\ex_{\bx\sim\mu^n}\hat{m}_\bx(\phi)\\
&=\ex_{\mu\sim\eta}\ex_{\bx\sim\mu^n}\max\limits_{g\in\cG}\ell^\mu_h(\phi,\hat{f}^\phi_\bx,g)-\max\limits_{g\in\cG}\hat{\ell}^\bx_h(\phi,\hat{f}^\phi_\bx,g)\\
&\le^{(a)}\ex_{\mu\sim\eta}\ex_{\bx\sim\mu^n}\max\limits_{g\in\cG}[\ell^\mu_h(\phi,\hat{f}^\phi_\bx,g)-\hat{\ell}^\bx_h(\phi,\hat{f}^\phi_\bx,g)]\\
&\le\ex_{\mu\sim\eta}\ex_{\bx\sim\mu^n}\max\limits_{f\in\cF}\max\limits_{g\in\cG}[\ell^\mu_h(\phi,f,g)-\hat{\ell}^\bx_h(\phi,f,g)]\\
&\le^{(b)}\epsilon_{gen,h}(\cF,\cG)
\end{align*}
where $(a)$ follows by observing that $\max_{g}[\theta(g)-\theta'(g)]\le\max_g\theta(g) - \max_g\theta'(g)$ for any functions $\theta,\theta'$, for the first inequality and $(b)$ follows from Lemma~\ref{lem:gen1}.
\end{proof}
\begin{proof}[\bf Proof of Lemma~\ref{lem:gen3}]
We will be using Slepian's lemma
\noindent Using Theorem 8(ii) from \cite{maurer2016benefit}, we get that
\begin{align}\label{eq:gen_phi}
\sup\limits_{\phi\in\Phi}\left[\ex_{\bx\sim\rho_h}\hat{m}_{\bx}(\phi)-\frac{1}{T}\sum\limits_{i}\hat{m}_{\bx^{(i)}}(\phi)\right]\le\frac{\sqrt{2\pi}}{T}G(S)+\sqrt{\frac{9\ln(2/\delta)}{2T}}
\end{align}
where $S=\{(\hat{m}(\phi)_{\bx_1},\dots,\hat{m}(\phi)_{\bx_T}):\phi\in\Phi\}$.
We bound the Gaussian average of $S$ using Slepian's lemma.
Define two Gaussian processes indexed by $\Phi$ as
\begin{align*}
X_\phi=\sum\limits_i\gamma_i\hat{m}(\phi)_{\bx^{(i)}}\ \text{and}\ Y_\phi=\frac{2K}{\sqrt{n}}\sum\limits_i\gamma_{ijk}\phi(s^i_j)_k
\end{align*}
For $\bx=\{(s_j,a_j,\tilde{s}_j,\bar{s}_j)\}$, consider 2 representations $\phi$ and $\phi'$,
\begin{align*}
(\hat{m}(\phi)_{\bx}-\hat{m}(\phi')_{\bx})^2&=(\min\limits_{f\in\cF}\max\limits_{g\in\cG}\hat{\ell}^{\bx}_h(\phi,f,g)- \min\limits_{f\in\cF}\max\limits_{g\in\cG}\hat{\ell}^{\bx}_h(\phi',f,g))^2\\
&\le(\sup\limits_{f\in\cF,g\in\cG}|\hat{\ell}^{\bx}_h(\phi,f,g)-\hat{\ell}^{\bx}_h(\phi',f,g)|)^2\\
&=\left(\sup\limits_{f\in\cF,g\in\cG}\left|\frac{1}{n}\sum\limits_j [K\pi^{\phi,f}(a_j|s_j)g(\tilde{s}_j)-K\pi^{\phi',f}(a_j|s_j)g(\tilde{s}_j)]\right|\right)^2\\
&= K^2\left(\sup\limits_{f\in\cF,g\in\cG}\left|\frac{1}{n}\sum\limits_j \left(f(\phi(s_j))_{a_j}-f(\phi'(s_j))_{a_j}\right)g(\tilde{s}_j)\right|\right)^2\\
&\le\frac{K^2}{n}\sup\limits_{f\in\cF}\sum\limits_j\left(f(\phi(s_j))_{a_j}-f(\phi'(s_j))_{a_j}\right)^2\\
&\le\frac{4K^2}{n}\sum\limits_{j}|\phi(s_j)-\phi'(s_j)|^2=\frac{4K^2}{n}\sum\limits_{j,k}(\phi(s_j)_k-\phi'(s_j)_k)^2
\end{align*}
where we prove the first inequality later, second inequality comes from $g$ being upper bounded by 1 and by Cauchy-Schwartz inequality, third inequality comes from the 2-lipschitzness of $f$.
\begin{align*}
\expect(X_\phi-X_{\phi'})&=\sum\limits_i(\hat{m}(\phi)_{\bx^{(i)}}-\hat{m}(\phi')_{\bx^{(i)}})^2\\
&\le\frac{4K^2}{n}\sum\limits_{i,j,k}(\phi(s^i_j)_k-\phi'(s^i_j)_k)^2=\expect(Y_\phi-Y_{\phi'})^2
\end{align*}
Thus by Slepian's lemma, we get
\begin{align*}
G(S)=\expect\sup\limits_{\phi\in\Phi} X_\phi\le\expect\sup\limits_{\phi\in\Phi} Y_\phi=\frac{2K}{\sqrt{n}}G(\Phi(\{s^i_j\}))
\end{align*}
Plugging this into Equation~\ref{eq:gen_phi} completes the proof.
To prove the first inequality above, notice that
\begin{align*}
\min\limits_{f\in\cF}\max\limits_{g\in\cG}\hat{\ell}^{\bx}_h(\phi,f,g)-\min\limits_{f\in\cF}\max\limits_{g\in\cG}\hat{\ell}^{\bx}_h(\phi',f,g)&=\hat{\ell}^{\bx}_h(\phi,f,g)-\hat{\ell}^{\bx}_h(\phi',f',g')\\
&\le\hat{\ell}^{\bx}_h(\phi,f',g'')-\hat{\ell}^{\bx}_h(\phi',f',g')\\
&\le\hat{\ell}^{\bx}_h(\phi,f',g'')-\hat{\ell}^{\bx}_h(\phi',f',g'')\\
&\le\sup\limits_{f\in\cF,g\in\cG}|\hat{\ell}^{\bx}_h(\phi,f,g)-\hat{\ell}^{\bx}_h(\phi',f,g)|
\end{align*}
By symmetry, we also get that $\min\limits_{f\in\cF}\max\limits_{g\in\cG}\hat{\ell}^{\bx}_h(\phi,f,g)-\min\limits_{f\in\cF}\max\limits_{g\in\cG}\hat{\ell}^{\bx}_h(\phi',f,g)\le \sup\limits_{f\in\cF,g\in\cG}|\hat{\ell}^{\bx}_h(\phi,f,g)-\hat{\ell}^{\bx}_h(\phi',f,g)|$.
\end{proof}
\begin{proof}[Proof of Lemma~\ref{lem:delta}]
Let $\bar{g}=\arg\max\limits_{g\in\cG}\left(\ex_{s\sim\nu_h^\pi}\Gamma_h^{\bpi}g(s)-\ex_{s\sim\nu_h^*}\Gamma_h^{\bpi}g(s)\right)$ and $g'=\arg\min\limits_{g\in\cG}|g-\Gamma_h^\bpi\bar{g}|_{(\nu_h^\pi+\nu_h^*)/2}$.
\begin{align*}
\max\limits_{g\in\cG}&\left(\ex_{s\sim\nu_h^\pi}\Gamma_h^{\bpi}g(s)-\ex_{s\sim\nu_h^*}\Gamma_h^{\bpi}g(s)\right)=\ex_{s\sim\nu_h^\pi}\Gamma_h^{\bpi}\bar{g}(s)-\ex_{s\sim\nu_h^*}\Gamma_h^{\bpi}\bar{g}(s)\\
&=\ex_{s\sim\nu_h^\pi}g'(s)-\ex_{s\sim\nu_h^*}g'(s)+\ex_{s\sim\nu_h^\pi}[\Gamma_h^{\bpi}\bar{g}(s)-g'(s)]+\ex_{s\sim\nu_h^*}[g'(s)-\Gamma_h^{\bpi}\bar{g}(s)]\\
&\le|\ex_{s\sim\nu_h^\pi}g'(s)-\ex_{s\sim\nu_h^*}g'(s)|+\ex_{s\sim\nu_h^\pi}[|g'(s)-\Gamma_h^{\bpi}\bar{g}(s)|]+\ex_{s\sim\nu_h^*}[|g'(s)-\Gamma_h^{\bpi}\bar{g}(s)|]\\
&\le\max_{g\in\cG}|\ex_{s\sim\nu_h^\pi}g(s)-\ex_{s\sim\nu_h^*}g(s)|+2\ex_{s\sim(\nu_h^\pi+\nu_h^*)/2}\left[|g'(s)-\Gamma_h^{\bpi}\bar{g}(s)]|\right]\\
&\le\Delta_h+2\epsilon_{be}^\pi
\end{align*}
\end{proof} 
\section{Data Set Collection Details}
\label{sec:data_collect}
\subsection{Dataset from trajectories}\label{subsec:dataset_bc}
Given $n$ expert trajectories for a task $\mu$, for each trajectory $\tau=(s_1,a_1,\dots,s_H,a_H)$ we can sample an $h\sim\gU([H])$ and select the pair $(s_h,a_h)$ from that trajectory\footnote{In practice one can use all pairs from all trajectories, even though the samples are not strictly i.i.d.}.
This gives us $n$ i.i.d. pairs $\{(s_j,a_j)\}_{j=1}^n$ for the task $\mu$.
We collect this for $T$ tasks and get datasets $\bx^{(1)},\dots,\bx^{(T)}$.

\subsection{Dataset from trajectories and interaction}\label{subsec:dataset_oo}
Given $2n$ expert trajectories for a task $\mu$, we use first $n$ trajectories to get independent samples from the distributions $\nu^*_{1,\mu},\dots,\nu^*_{H,\mu}$ respectively for the $\bar{s}$ states in the dataset.
Using the next $n$ trajectories, we get samples from $\nu^*_{0,\mu},\dots,\nu^*_{H-1,\mu}$ for the $s$ states in the dataset, and for each such state we uniformly sample an action $a$ from $\actions$ and then get a state $\tilde{s}$ from $P_{s,a}$ by resetting the environment to $s$ and playing action $a$.
We collect this for $T$ tasks and get datasets $\bX^{(i)}=\{\bx^{(i)}_1,\dots,\bx^{(i)}_H\}$ for every $i\in[T]$, where each dataset $\bx^{(i)}_h$ a set of $n$ tuples obtained level $h$.
Rearranging, we can construct the datasets $\bX_h=\{\bx^{(1)}_h,\dots,\bx^{(T)}_h\}$. 
\section{Experiment Details}
\label{sec:exp_details}
For the policy optimization experiments, we use 5 random seeds to evaluate our algorithm.
We show the results for 1 test environment as the results for other test environments are also showing the algorithm works but the magnitude of reward might be different, so we do not average the numbers over different test environments.
\paragraph{Environment Setup}
We first describe the construction of the NoisyCombinationLock environment.
The state space is $\R^{50}$. Each state $s$ is in the form of
$[s_\text{noise}, s_\text{index}, s_\text{real}]$,
where $s_\text{noise} \in \R^{10}$ is sampled from $\mathcal{N}(\mathbf{0}, w \mathbf{I})$,
$s_\text{index} \in \R^{10}$ is an one-hot vector indicating the current step
and $s_\text{real} \in \R^{10}$ is sampled from
$\mathcal{N}(\mathbf{0}, w \mathbf{I})$.
The constant $w$ is set to $\sqrt{0.05}$ such that $\|s_\text{real}\|$ has expected norm of 1.
The action space is $\{-1, 1\}$.
Each MDP is parametrized by a vector $\textbf{c}^* \in \{-1, 1\}^{20}$, which determines the optimal action sequence.
We use different $\textbf{a}^*$ to define different
environments.
The transition model is that: Let  $s = [s_\text{noise}, s_\text{index}, s_\text{real}]$ be the current state and $a$ be the action. If $s_\text{index} = e_i$ for some $i$ and
$\textbf{c}^*_i a s_{\text{real} i} > 0$, then $s_\text{index}' = e_{i + 1}$.
Otherwise $s_\text{index}'$ will be all zero.
$s'_\text{noise}$ and $s'_\text{real}$ will always be sampled from the Gaussian distribution.
The reward is 1 if and only if $s_\text{index}$ is not a zero vector, otherwise it's 0.
Note that once $s_\text{index}$ is all zero, it will not change and the reward will
always be $0$. The maximum horiozn is set to $20$ and therefore, the optimal
policy has return 20. The initial $s_\text{index}$ is always $e_1$.

The SwimmerVelocity environment is similar to goal velocity experiments in \citep{finn2017model},
and is based on the
Swimmer environment in OpenAI Gym \citep{Gym}. The only difference is the reward function,
which is now defined by $r(s) = |v - v_\text{goal}|$, where $v$ is the current velocity of the agent and $v_\text{goal}$ is the goal velocity.
The state space is still $\R^8$. The original action space in Swimmer is $\R^2$,
and we discretize the action space, such that each entry can be only one of $\{-1,
-0.5, 0, 0.5, 1\}$. We also reduce the maximum horizon from 1000 to 50.

\paragraph{Experts} For NoisyCombinationLock, the demonstrations are generated by the optimal policy, which has access to the hidden vector $\textbf{c}^*$. For SwimmerVelocity, we trained the experts for 1 million steps by PPO \citep{PPO} to make sure it converges with code from \citet{baselines}.

\paragraph{Architecture} For all of our experiments in Figure~\ref{fig:theory} and \ref{fig:rl}, the function $\phi$ is parametrized by $\phi(x) = \relu{Wx + b}$ where $W, b$ are learnable parameters of a linear layer and $\relu{\cdot}$ is the ReLU activation function. However, the number of hidden units might vary. Note that in the experiments of verifying our theory (Figure~\ref{fig:theory}), we train a policy (and the representation) at each step so the dimension of representation is smaller.
See Table~\ref{table:hyperparameters} for our choice of hyperparameters.

\begin{table}[h!]
    \centering
    \begin{tabular}{|c c c c|}
     \hline
      & BC (Figure~\ref{fig:theory}, left two) & OA (Figure~\ref{fig:theory}, right two) & RL (Figure~\ref{fig:rl}) \\
     \hline
     NoisyCombinationLock & 5 & 5 & 40 \\
     SwimmerVelocity & 100 & 20 & 100 \\
     \hline
    \end{tabular}
    \caption{Number of hidden units for different experiments. }
    \label{table:hyperparameters}
\end{table}

\paragraph{Optimization} All optimization, including training $\phi, \pi$ and behavior cloning baseline, is done by Adam \citep{kingma2014adam} with learning rate 0.001 until it converges, except NoisyCombinationLock in policy optimization experiments in Figure~\ref{fig:rl} where we use learning rate 0.01 for faster convergence.
To solve Equation \eqref{eq:L_hat_bc} and \eqref{eq:oa_loss}, we build a joint loss over $\phi$ and all $f$'s in each task,
\begin{align}\label{eq:appendix-exp}
    \mathcal{L}(\phi, f_1, \dots, f_T) = \frac{1}{nT}\sum\limits_{t=1}^{T} \sum\limits_{j=1}^n-\log(\pi^{\phi, f_t}(s^t_j)_{a^t_j}).
\end{align}
Then we minimize $\mathcal{L}(\phi, f_1, \dots, f_T)$ and obtain the optimal $\phi$.

\end{document}